\documentclass[letterpaper,journal]{IEEEtran}
\usepackage{iftex}
\ifPDFTeX
  \usepackage[T1]{fontenc}
\fi
\ifXeTeX
  \usepackage{fontspec}
\fi
\usepackage{amsmath,amsfonts}
\usepackage{algorithm}
\usepackage{array}
\usepackage[caption=false,font=normalsize,labelfont=sf,textfont=sf]{subfig}
\usepackage{textcomp}
\usepackage{stfloats}
\usepackage{url}
\usepackage{verbatim}
\usepackage{graphicx}
\usepackage{cite}

\usepackage{graphicx}%
\usepackage{multirow}%
\usepackage{amsmath,amssymb,amsfonts}%
\usepackage{amsthm}%
\usepackage{mathrsfs}%
\usepackage{xcolor}%
\usepackage{makecell}
\usepackage{booktabs}
\usepackage{textcomp}%
\usepackage{manyfoot}%
\usepackage{booktabs}%
\usepackage{algorithm}%
\usepackage{algorithmicx}%
\usepackage{algpseudocode}%
\usepackage{listings}%
\newtheorem{lemma}{Lemma}

\usepackage{amsmath, amssymb, amsthm}
\usepackage{tikz}
\usetikzlibrary{shapes.geometric, arrows.meta, positioning, fit, shadows.blur}

\newtheorem{theorem}{Theorem}

\begin{document}

\title{Euclidean Crowding–Driven Adaptive Euclidean–Hyperbolic Graph Learning}

\author{
Hao~Chen,
Jian~Sun,
Chuan~Shi,
and~Hongying~Zhang%
\IEEEcompsocitemizethanks{
\IEEEcompsocthanksitem
Hao Chen, Jian Sun, and Hongying Zhang are with the School of Mathematics and Statistics, Xi'an Jiaotong University, Xi'an 710049, China
(e-mail: haochen@stu.xjtu.edu.cn; jiansun@xjtu.edu.cn; zhyemily@mail.xjtu.edu.cn).

\IEEEcompsocthanksitem
Chuan Shi is with the Beijing Key Lab of Intelligent Telecommunications Software and Multimedia, Beijing University of Posts and Telecommunications, Beijing 100876, China
(e-mail: shichuan@bupt.edu.cn).

(Corresponding author: Hongying Zhang.)}
}


\maketitle
\begin{abstract}
Hyperbolic graph representation learning has emerged as a powerful paradigm for modeling hierarchical data. Unlike Euclidean space, which exhibits polynomial volume growth, hyperbolic space grows exponentially with radius, naturally accommodating tree-like structures without the geometric distortion known as Euclidean crowding. Despite strong empirical performance, the theoretical mechanisms underlying hyperbolic embeddings remain poorly understood. In this work, we provide a geometric quantifcation research for Euclidean crowding, a generalization-theoretic explanation for the superiority of hyperbolic space on hierarchical data, and further develop an adaptive Euclidean-hyperbolic fusion learning framework toward bridging this gap. Specifically, we \textbf{first} establish a geometric quantification of Euclidean crowding by analyzing layer-wise minimum separation in node embeddings, proving that fixed-dimensional Euclidean embeddings of tree-like graphs inevitably exhibit layer-wise crowding that leads to weaker margin-based generalization bounds. \textbf{Second}, we provide a generalization-bound for the superiority of hyperbolic geometry on hierarchical data. \textbf{Third}, we propose EH-Fusion, a geometry-adaptive framework that lifts Euclidean GNN representations into hyperbolic space via exponential mapping and adaptively fuses both geometries based on node-level neighborhood expansion patterns. Extensive experiments validate our theoretical insights and demonstrate the effectiveness of EH-Fusion on benchmark datasets.
\end{abstract}
\begin{IEEEkeywords}
Graph Learning, Hyperbolic Space, Hyperbolic Graph Neural Network, Generalization gap.
\end{IEEEkeywords}




\section{Introduction}




Euclidean geometry has long been the dominant representation space in machine learning\cite{bert,vit,clip}, due to its strong correspondence with human spatial intuition and well-established mathematical foundations. In particular, graph representation learning has predominantly relied on Euclidean embeddings\cite{gcn,gat}. However, an increasing body of evidence suggests that Euclidean geometry may not align well with the latent structures of many real-world datasets.

This mismatch is particularly evident in hierarchical or tree-like data, which arise in a wide range of domains, including concept taxonomies\cite{concept-tax,NIPS2017_poincare}, fine-grained image\cite{hyper-fine-grained}, RNA-seq data\cite{cellRNA-poincare}, protein interaction data\cite{proterin-hierarchical} and biological classification systems\cite{biological-hyper}. Such data typically exhibit exponential branching patterns. 
In contrast, the volume of Euclidean space grows only polynomially, leading to a fundamental geometric mismatch. As a result, Euclidean embeddings are prone to the so-called \emph{crowding phenomenon}\cite{crowding-phenomenon-acl2023}, where nodes at deeper layers become increasingly close to each other.  

 Hyperbolic geometry provides a natural alternative. Due to its exponential volume growth, hyperbolic space matches the expansion of hierarchical structures and admits low-distortion embeddings of tree-like data\cite{NIPS2017_poincare}. 
 In addition, prior studies~\cite{lgcn,hgat,shi} have suggested that hyperbolic space can also serve as a natural latent space for graphs with scale-free (power-law) structures and for heterogeneous information networks that exhibit similar power-law properties.
This observation has motivated a large body of work on deep hyperbolic representation learning across computer vision \cite{Yan_2021_CVPR,hcl,hyper_seg}, natural language processing \cite{NIPS2017_poincare,bert_probing}, and multi-modal learning \cite{pmlr-v202-hclip}. 
Beyond representation learning, hyperbolic geometry has also been leveraged for graph generation. Most recently, Bu et al. \cite{ggball} proposed GGBall, a hyperbolic generative framework operating in the Poincaré ball, which combines a hyperbolic autoencoder with Riemannian flow to generate hierarchical graphs in a one-shot manner. Their empirical results demonstrate that hyperbolic latent spaces significantly improve generation quality on tree-like and molecular graphs, further confirming the practical advantages of hyperbolic geometry.
\textbf{Despite these empirical successes, the mechanisms underlying the advantages of hyperbolic geometry and the intrinsic limitations of Euclidean geometry in modeling hierarchical data remain underexplored.} 
 Existing explanations are largely based on global distortion\cite{distortion,distortion_lilun} or qualitative geometric intuition\cite{NIPS2017_poincare}, and do not provide a quantitative characterization of how Euclidean crowding arises or how it affects representation learning.

\textbf{In this work, we provide a geometric analysis of hierarchical graph embeddings in Euclidean spaces.} Rather than analyzing global metric distortion, we study the \emph{layer-wise minimum separation} between node embeddings, providing a complementary perspective to the distortion-based theories of hyperbolic embeddings.
 Our analysis shows that, under fixed embedding dimensions, Euclidean representations of tree-like graphs inevitably exhibit \emph{layer-wise crowding}, where the separation between nodes within the same layer shrinks as depth increases. 
Furthermore, we show that  Euclidean hierarchical embeddings induce an intrinsic geometric tension between preserving parent–child distances and maintaining separation between non-adjacent nodes. Enforcing separation between non-adjacent nodes results in increased radial distances and stretching parent–child edges, whereas preserving tree-path distances inevitably compresses the separation between non-adjacent nodes. This phenomenon gives rise to a fundamental \emph{stretch–squeeze trade-off}, revealing a structural limitation of Euclidean embeddings for tree-like graphs.

\textbf{Beyond the geometric perspective, we establish a direct connection between representation geometry and learning theory.} By linking the separation gap between positive and negative pairs to the functional margin of pairwise distance-based models, we show that Euclidean crowding reduces the achievable margin and consequently leads to weaker margin-based generalization guarantees. \textbf{To our knowledge, this is the first work that quantitatively characterizes layer-wise crowding in  Euclidean embeddings and connects it to generalization behavior.}

\begin{figure*}[htb]
	\centering
	\includegraphics[width=\linewidth]{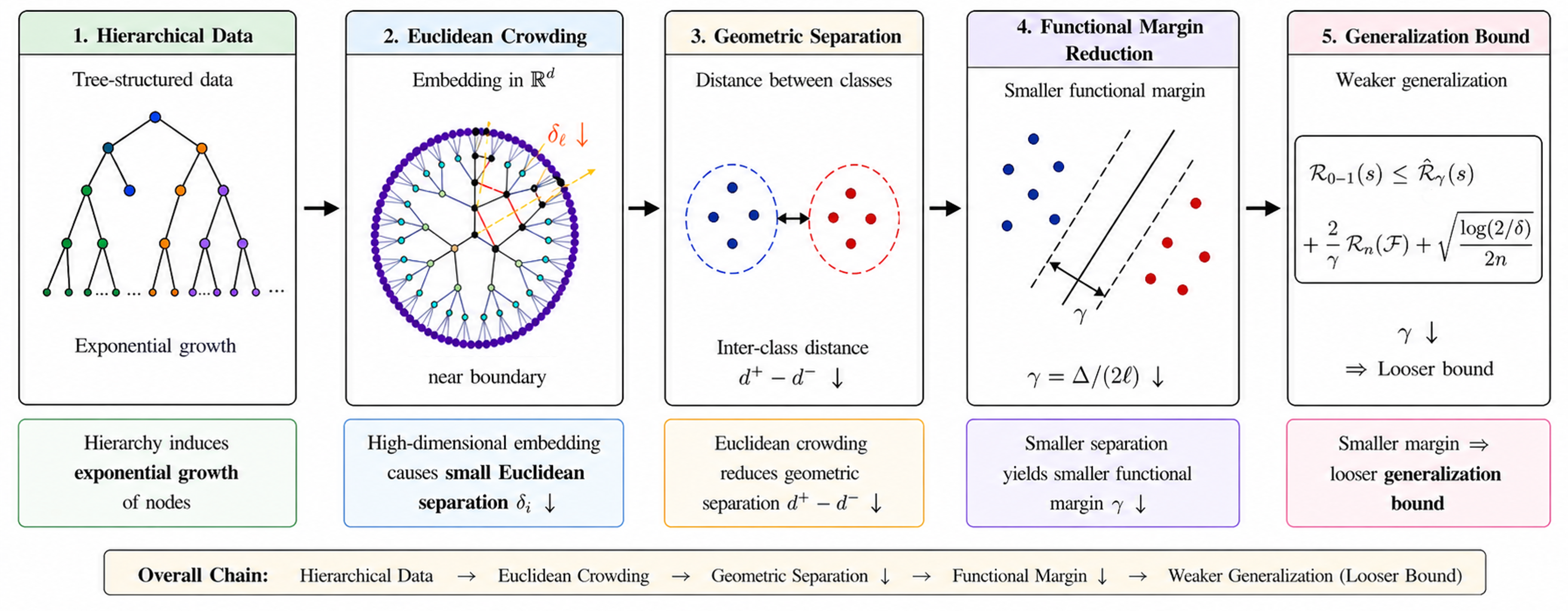}
    \caption{Euclidean crowding and its connection to generalization. In hierarchical data, the exponential growth of nodes with depth compresses pairwise distances in Euclidean space, reducing the geometric separation $\Delta$ and the induced functional margin $\gamma$. This margin reduction leads to a looser margin-based generalization bound, whereas hyperbolic space can better preserve hierarchical distances with lower distortion.}
\end{figure*}

Motivated by this observation, we first introduce a lightweight Euclidean-to-hyperbolic mapping module (GnnE2H), which lifts Euclidean graph representations into hyperbolic space, such as the Lorentz model or the Poincaré ball, via a simple exponential mapping layer. Experimental results on the strongly hyperbolic dataset Disease demonstrate that adding a single hyperbolic mapping layer is sufficient to unlock the advantages of hyperbolic geometry. Specifically, when applied to GCN, GCN2L achieves 94.9\% AUC, outperforming the tangent-space approximation method HGCN (89.9\%) by a substantial margin. When applied to GAT, GAT2L reaches 97.5\% AUC, surpassing the fully hyperbolic network FHNN (96.0\%) by 1.5 percentage points. These results not only support our theoretical analysis but also align with the principle of Occam's razor, yielding a lightweight yet effective architecture that captures the benefits of hyperbolic geometry without requiring a fully hyperbolic network.

Moreover, recent studies have shown that real-world graphs often exhibit mixed geometric characteristics that cannot be fully described by a single manifold such as hyperbolic space\cite{hyperbolic-mix}. This challenge reflects an underlying geometric mismatch: the advantages of hyperbolic space diminish for graphs with weaker hyperbolicity or less tree-like structures, where imposing a fixed hyperbolic space may lead to performance degradation. To address this issue, CurvGAD \cite{curvgad} proposes a mixed-curvature graph autoencoder that explicitly models geometric irregularities through curvature-aware encoding. In particular, it incorporates Ollivier--Ricci curvature, along with curvature reconstruction and Ricci-flow regularization to discern geometry-dependent patterns. \textbf{While effective, such an explicit curvature modeling strategy relies on curvature estimation and curvature-driven reconstruction, which introduces additional computational overhead.}
To address the risk of geometric mismatch and implicitly capture structural properties for mixed-manifold modeling, building upon the GnnE2H mapping module, we further propose EH-Fusion, an adaptive Euclidean--hyperbolic fusion framework that leverages curvature via node-level neighborhood expansion patterns to dynamically balance Euclidean and hyperbolic geometries with low overhead and high portability.

Our method exploits the complementary roles of Euclidean and hyperbolic geometries in modeling graphs with varying degrees of hierarchy. Hyperbolic space provides a natural inductive bias for tree-like graphs through its exponential expansion property, whereas Euclidean space remains a reliable alternative when hierarchical structure is weak or less pronounced. In contrast, positively curved spherical geometry is more suitable for structures dominated by cycles or cliques, which fall outside the flat-to-hierarchical regime considered in this work. Therefore, we focus on Euclidean--hyperbolic complementarity for hierarchy-aware graph learning.

Overall, this work provides a unified geometric and learning-theoretic perspective on graph representation learning for hierarchical data. By quantifying layer-wise crowding in Euclidean embeddings for tree-like data and linking it to margin-based generalization, we offer a principled explanation for the advantages of hyperbolic representations and derive models from these theoretical insights.

Our main contributions are summarized as follows:

\begin{itemize}
\item \textbf{Quantifying Euclidean crowding in hierarchical representations.}
We show that fixed-dimensional Euclidean representations inevitably suffer from layer-wise crowding when embedding tree-like graphs, by providing a quantitative geometric characterization of this phenomenon. Our analysis further reveals an intrinsic stretch--squeeze trade-off between preserving parent--child distances and maintaining separation among non-adjacent nodes.

\item \textbf{Linking Euclidean crowding to margin-based generalization.}
We establish a principled connection between Euclidean crowding and margin-based generalization  by relating layer-wise separation to the functional margin of distance-based models. Our analysis shows that Euclidean crowding inherently limits the achievable margin, leading to weaker generalization guarantees.

\item \textbf{Geometry-Adaptive Graph Learning via Euclidean--Hyperbolic Fusion.}
Motivated by the above insights, we propose EH-Fusion, an adaptive Euclidean--hyperbolic framework that dynamically balances  Euclidean and hyperbolic geometries according to node-level neighborhood expansion patterns, enabling robust learning across graphs with heterogeneous geometric structures.
\end{itemize}


\section{Related Work}

\subsection{Deep Hyperbolic Representation Learning}
Hyperbolic geometry has recently attracted increasing attention in deep representation learning for modeling hierarchical structures. In visual representation learning, Yan et al.~\cite{Yan_2021_CVPR} introduced hyperbolic metrics and hierarchical clustering to capture multi-level semantic relations, addressing the difficulty of modeling latent hierarchical structures in tree-like data. In language representation learning, Nickel and Kiela~\cite{NIPS2017_poincare} proposed Poincar\'e embeddings, demonstrating that hyperbolic geometry effectively encodes latent hierarchical structures with enhanced representational capacity and structural fidelity. In multimodal representation learning, Desai et al.~\cite{pmlr-v202-hclip} developed a hyperbolic image--text contrastive framework, leveraging the geometric properties of hyperbolic space to model multi-level semantic organization in vision--language representations.

A common design paradigm in these approaches is to introduce a hyperbolic realization layer on top of Euclidean backbone models to leverage the metric advantages of hyperbolic space for hierarchical representation, while the underlying architectures remain Euclidean.
This paradigm suggests that introducing a simple exponential mapping layer enables deep Euclidean encoders to be effectively optimized under hyperbolic metrics, yielding high-quality hyperbolic embeddings. These observations indicate that the effectiveness of hyperbolic representations largely stems from the  metric geometry itself.

\subsection{Hyperbolic Graph Representation Learning}
\textbf{Recent advances in hyperbolic graph representation learning have followed a progressively deepening geometric modeling paradigm.} HGNN~\cite{hgnn} introduced a tangent-space learning framework, where graph aggregation is performed in the Euclidean tangent space, followed by projection onto hyperbolic manifolds via exponential mappings. Building upon this idea, HGCN~\cite{hgcn} instantiated this framework with GCN-style architectures while preserving inductive capability. Subsequently, FHNN~\cite{fhnn} proposed a fully hyperbolic framework, where feature transformations, neighborhood aggregation, and nonlinearities are defined entirely within hyperbolic manifolds. By performing all operations directly in hyperbolic space, FHNN achieves stronger geometric consistency and improved expressiveness for hierarchical graph structures. Overall, hyperbolic graph representation learning has evolved from tangent-space approximated GNNs to intrinsic hyperbolic graph operations, reflecting a transition from geometry-assisted representations to fully geometry-consistent ones.

\textbf{Taken together, existing studies reveal two distinct paradigms of hyperbolic modeling.} In deep representation learning, hyperbolic geometry is introduced via a mapping layer that lifts Euclidean representations onto hyperbolic space and enables optimization under hyperbolic  metrics. In graph neural networks, recent studies have explored deeper integration of hyperbolic geometry throughout the entire network architecture.
\textbf{This naturally raises the questions of what fundamentally drives the effectiveness of hyperbolic representations and what principles can guide our selection among different pathways.} In this work, through margin-based generalization analysis and extensive experiments, we show that much of this advantage stems from the underlying metric geometry.
 

\subsection{Geometric Perspective on Euclidean Crowding}

Hyperbolic geometry is widely recognized as a natural representation space for hierarchical or tree-like data. This advantage is often attributed to the exponential volume growth of hyperbolic space, which aligns with the exponential expansion of hierarchical structures. In contrast, the polynomial volume growth of
Euclidean space creates a geometric mismatch, leading to the
crowding phenomenon\cite{NIPS2017_poincare,crowding-phenomenon-acl2023,crowding-phenomenon-2023KDD,crowding1} among nodes in deeper layers.
Several studies have established the theoretical advantages of hyperbolic embeddings. In particular, the notion of $\delta$-hyperbolicity introduced by Gromov characterizes the tree-likeness of graphs, where trees are 0-hyperbolic in the sense of Gromov and therefore naturally align with hyperbolic geometry \cite{Gromov1987}. Moreover, tree metrics admit embeddings into the hyperbolic plane with only $O(1)$ distortion, whereas Euclidean embeddings typically incur $O(\log n)$ distortion \cite{distortion_lilun}.
Nevertheless, existing studies primarily focus on global distortion properties, while the mechanisms underlying this geometric mismatch and the resulting crowding phenomenon remain underexplored. In particular, how Euclidean crowding arises in hierarchical representations and how it affects learning performance has not been quantitatively characterized.
 In this work, we provide a complementary geometric perspective by analyzing the layer-wise minimum separation between node embeddings. Inspired by sphere packing arguments\cite{kabatiansky1978bounds,cohn2012sphere},which study the maximal number of points that can be placed within a sphere under a fixed minimum separation, we instead analyze the opposite scenario, where all nodes are embedded within the sphere and characterize how the minimum separation decays with tree depth. 
 We show that fixed-dimensional Euclidean embeddings of tree-like graphs inevitably exhibit layer-wise crowding. Furthermore, by relating geometric separation to the functional margin of pairwise distance-based models, we establish that such crowding leads to weaker margin-based generalization guarantees, providing a theoretical perspective on the advantages of hyperbolic representations for hierarchical data. 


\section{Preliminary}

\subsection{Riemannian Geometry}

\paragraph{\textbf{Manifold}}
An $n$-dimensional manifold $\mathcal{M}$ is a topological space that is locally homeomorphic to $\mathbb{R}^n$.
\paragraph{\textbf{Tangent Space}}
For any point $x \in \mathcal{M}$, the tangent space $T_x\mathcal{M}$ is the vector space of all tangent vectors at $x$,and is isomorphic to $\mathbb{R}^n$.
\paragraph{\textbf{Riemannian Manifold}}
A Riemannian manifold $(\mathcal{M}, g)$ is a smooth manifold equipped with a Riemannian metric $g = (g_x)_{x \in \mathcal{M}}$, where each $g_x : T_x\mathcal{M} \times T_x\mathcal{M} \rightarrow \mathbb{R}$ is an inner product varying smoothly with $x$.
This metric induces geometric notions such as length and curvature.
\paragraph{\textbf{Geodesics and Distance}}
Let \( \mathcal{M} \) be a Riemannian manifold, and let \( \gamma : [a, b] \to \mathcal{M} \) be a piecewise smooth curve on the manifold. The length of the curve \( \gamma \) is defined as: \(
L(\gamma) = \int_a^b \|\dot{\gamma}(t)\|_g \ dt\).  For any two points \(x, y \in \mathcal{M}\), the geodesic distance between them is defined as: \(
d_g(x, y) = \inf \left\{ L(\gamma) \right\} \), where $\gamma(a) = x \text{ and } \gamma(b) = y$. The geodesic distance generalizes the notion of Euclidean distance  to curved Riemannian manifolds, where a geodesic is defined as the shortest path between two points. In particular, for certain manifolds such as hyperbolic manifolds, geodesics can be expressed in a closed  form.
\paragraph{\textbf{Exponential and Logarithmic Maps}}
For each point \( x \in \mathcal{M} \) and vector \( v \in T_x\mathcal{M} \), there exists a unique geodesic \( \gamma_v \) such that \( \gamma_v(0) = x \) and \( \dot{\gamma}_v(0) = v \). The Riemannian exponential map is defined as \(
\exp_x(v) = \gamma_v(1)
\), which maps the tangent vector \( v \)  to a point on the manifold \( \mathcal{M} \). The logarithmic map \( \log_x : \mathcal{M} \to T_x\mathcal{M} \) is the inverse of \( \exp_x \), which pulls  the points on the manifold back to the tangent space at \( x \).

\subsection{Hyperbolic Manifold}

A hyperbolic manifold is a Riemannian manifold with constant negative curvature, whose volume grows exponentially. This geometric property makes it an ideal embedding space for hierarchical data. Hyperbolic space is commonly represented by two isometric models, the Poincaré ball and the Lorentz model, both of which preserve its intrinsic geometry and are mutually transformable.

\paragraph{\textbf{Poincaré Space}}
The $d$-dimensional Poincaré ball model with negative curvature $-c\ (c>0)$ is denoted by $\mathcal{B}_c^d$, and its domain is given by: \(
\mathcal{B}_c^d = \left\{ x \in \mathbb{R}^d : \|x\|^2 < {1}/{c} \right\}
\), where $\|\cdot\|$ denotes the Euclidean norm. The origin in the Poincaré ball model corresponds to the Euclidean origin. The hyperbolic distance between two points \( g_u, g_v \in \mathcal{B} \) is defined as:
\begin{equation}
\small
    d_c^\mathcal{B}(g_u, g_v) =\frac{1}{\sqrt{c}} \operatorname{arcosh}\left( 1 + \frac{2c \| g_u - g_v \|^2}{(1 - c\| g_u \|^2)(1 - c\| g_v \|^2)} \right).
\end{equation}
For $x \in \mathcal{T}_0\mathcal{B}_c^d=\mathbb{R}^d$ and $z \in \mathcal{B}_c^d$, the exponential and the logarithmic map between the hyperbolic space $\mathcal{B}_c^d$  and the tangent space $\mathcal{T}_0\mathcal{B}_c^d$  is defined as  :
\begin{align}
exp_0^\mathcal{B}(x) = \tanh\left(\sqrt{c} \|x\|\right) \frac{x}{\sqrt{c} \|x\|} \\
log_0^\mathcal{B}(z) = \text{artanh} \left(\sqrt{c} \|z\|\right) \frac{z}{\sqrt{c} \|z\|}
\end{align}
where the  origin $\mathbf{0}$   is $(0,...,0) \in \mathcal{B}_c^d$.

\paragraph{\textbf{Lorentz Space}}
The  $d$-dimensional Lorentzian space with the negative curvature $-1/K\ (K>0)$ is denoted by \( \mathcal{L}_K^d \) , and its domain is given by \(
\mathcal{L}_K^d = \left\{ x \in \mathbb{R}^{d+1} \mid \langle x, x \rangle_\mathcal{L} = -K, x_0 > 0 \right\}
\), where the Lorentzian inner product is defined as \( \langle x, y \rangle_\mathcal{L} := -x_0 y_0 + \sum_{i=1}^{d} x_i y_i =-x_0 y_0 +x_{space} \cdot y_{space}\). The Lorentzian distance between two points \( x, y \in \mathcal{L}_K^d \) is given by:
\begin{equation}
d_K^\mathcal{L}(x, y) = \sqrt{K} \operatorname{arcosh}(-\langle x, y \rangle_\mathcal{L}/{K} ).
\end{equation}
For $v \in \mathbb{R}^d ,z \in \mathcal{L}_K^d$, the exponential and  the logarithmic map is defined as:

{\small
\begin{align}
 exp_0^\mathcal{L}(v)& = \left( \sqrt{K} \cosh \left( \frac{\|v\|}{\sqrt{K}} \right), \sqrt{K} \sinh \left( \frac{\|v\|}{\sqrt{K}} \right) \frac{v}{\|v\|} \right)\\   log_0^\mathcal{L}(z)& = \left( 0, \sqrt{K}  \text{arcosh} \left( \frac{x_0}{\sqrt{K}} \right)  \frac{x_{space}}{\|x_{space}\|} \right)
\end{align}
}
where the  origin $\mathbf{0}$  is   $(\sqrt{K},0,...,0) \in \mathcal{L}_K^d$.

\subsection{Metric Space}

A metric space \( (X, d) \) is a set \( X \) equipped with a \textit{distance function}  \( d: X \times X \to \mathbb{R} \) that satisfies the following  conditions for all \( x, y, z \in X \):

\begin{enumerate}

    \item {Non-negativity:} $d(x,y) \ge 0$.
    \item {Identity of indiscernibles:} $d(x,y)=0$ if and only if $x=y$.
    \item {Symmetry:} $d(x,y)=d(y,x)$.
    \item {Triangle inequality:} $d(x,z) \le d(x,y)+d(y,z)$.
\end{enumerate}

\textbf{Euclidean hyperball space.} The $d$-dimensional hyperball of radius $r$ is defined as \(
B^{d}(r) := \bigl\{ \mathbf{x} \in \mathbb{R}^d  : \|\mathbf{x}\| < r \bigr\},
\) which is equipped with the Euclidean $\ell_2$  distance metric  \(d_2(\mathbf{x},\mathbf{y}) := \|\mathbf{x}-\mathbf{y}\|_2, 
\forall\, \mathbf{x},\mathbf{y} \in B^{d}(r).\)
The  volume of  \(
B^{d}(r) \)  is given by \(
\mathrm{Vol} \big( B^d(r)\big) =\frac{\pi^{d/2} r^d}{\Gamma\left(1 + \frac{d}{2}\right)},
\) where $\Gamma(\cdot)$ denotes the Gamma function, satisfying $\Gamma(n)=(n-1)!$ for positive integers $n$.

\textbf{Tree metric space.} Let $T$ be a rooted tree of depth $L \ge 1$, where each internal node has  $k \ge 2$ children on average.  We denote by $|\mathcal{V}_l|$ the number of nodes at depth $l$. Edges connecting nodes at depth $i-1$ to depth $i$ are assigned positive weights $d_i$ with \(
d_1 \ge  d_2 \ge  \dots \ge d_L>0
\), reflecting increasing similarity among nodes at deeper levels.  We denote a node at depth $k$ by $v_{k,j_k}$, where $j_k$ indexes its position at  depth $k$.
For any two nodes $v_{k,j_k}$ and $v_{m,j_m}$, let $v_{h,j_h}$ denote their lowest common ancestor (LCA). Since $T$ is a tree, there exists a unique shortest path between $v_{k,j_k}$ and $v_{m,j_m}$ that ascends from $v_{k,j_k}$ to $v_{h,j_h}$ and then descends to $v_{m,j_m}$: \(
v_{k,j_k} \to v_{k-1,j_{k-1}} \to \dots \to v_{l,j_l} 
\to \dots \to v_{m-1,j_{m-1}} \to v_{m,j_m}.
\) The induced tree metric $d_T(\cdot,\cdot)$ is defined as the sum of edge weights along this path: \(
d_T\bigl(v_{k,j_k}, v_{m,j_m}\bigr) 
= \sum_{i=h+1}^{k} d_i \;+\; \sum_{i=h+1}^{m} d_i.
\) Specifically, let $u$ and $v$ be two leaf nodes at depth $l$ whose lowest common ancestor is the root. Then their tree distance is given by
\begin{equation}
d_T(u,v)=2\sum_{i=1}^{l} d_i .
\label{eq:app_tree_dist}
\end{equation}

\subsection{Gromov $\delta$-Hyperbolicity}

To quantify the hierarchical structure of graph data, we adopt Gromov $\delta$-hyperbolicity\cite{delta1,Gromov1987} as a metric characterization of tree-likeness. 
Unlike curvature, which describes local bending behavior, $\delta$-hyperbolicity captures how closely the global distance structure resembles a tree metric. 
This directly connects to our focus on hierarchical expansion and the corresponding benefits of hyperbolic representations.

\paragraph{\textbf{Gromov $\delta$-Hyperbolicity}}
Let $(X,d)$ be a metric space. For any four points $x,y,z,t \in X$, define the three pairwise distance sums:
\begin{equation}
\begin{aligned}
S_1 &= d(x,y)+d(z,t),\\
S_2 &= d(x,z)+d(y,t),\\
S_3 &= d(x,t)+d(y,z).
\end{aligned}
\end{equation}
The metric space $(X,d)$ is said to be $\delta$-hyperbolic if, for every quadruple of points, the two largest values among $\{S_1,S_2,S_3\}$ differ by at most $2\delta$, i.e.,
\begin{equation}
\max\{S_1,S_2,S_3\}
-
\operatorname{mid}\{S_1,S_2,S_3\}
\le 2\delta ,
\end{equation}
where $\operatorname{mid}(\cdot)$ denotes the second-largest value.

Intuitively, $\delta$ measures how close a metric space is to a tree metric. A smaller $\delta$ indicates stronger hierarchical organization and greater tree-likeness, favoring hyperbolic geometry as an inductive bias. Conversely, a relatively large $\delta$ suggests weaker tree-like structure, where Euclidean geometry becomes a natural complement. In particular, trees are $0$-hyperbolic.

\paragraph{\textbf{Why $\delta$-Hyperbolicity Instead of Curvature}}
Although hyperbolic geometry is classically associated with negative curvature, curvature and Gromov $\delta$-hyperbolicity offer different geometric perspectives. Curvature characterizes local geometric behavior, such as local bending, geodesic divergence, or neighborhood connectivity patterns. In comparison, $\delta$-hyperbolicity provides a global geometric measure of tree-likeness and hierarchical organization. 
Specifically, it characterizes how closely the global distance structure of a metric space resembles a tree metric, thereby reflecting the strength of hierarchical structure in graph data.
The advantage of hyperbolic representations is closely related to the exponential expansion property of hyperbolic space, which enables hierarchical structures to be represented with improved separation and reduced crowding. Therefore, our analysis focuses on whether the overall graph geometry exhibits strong global hierarchy and tree-like expansion, rather than relying only on the curvature behavior of local regions.
For this reason, we use Gromov $\delta$-hyperbolicity as a structural indicator for assessing whether graph data are likely to benefit from hyperbolic representation learning. 


\subsection{Problem Setting and Learning Objective}

We consider an undirected graph $G=(\mathcal V,\mathcal E)$ , where each node $i\in\mathcal V$ is associated with an initial Euclidean feature vector $x_i^{0,E}\in\mathbb R^d$, and each edge $(u,v)\in\mathcal E$ represents to a direct parent--child (is-a) relation. The goal of hyperbolic graph representation learning is to learn an mapping
$f:G\rightarrow\mathbf Z\in\mathbb H^{|\mathcal V|\times d}$ that preserves the hierarchical structure of the graph.
 
To evaluate how well the learned embeddings preserve hierarchical structure, we employ a distance-based link prediction task, which directly examines whether parent–child pairs are embedded closer than non-adjacent pairs.
The positive pair set $\mathcal D^+\subset\mathcal V\times\mathcal V$  with label $y_{uv}=1$ is given by all node pairs $(u,v)\in\mathcal E$, corresponding to direct parent--child relations. The negative pair set $\mathcal D^-\subset\mathcal V\times\mathcal V$  with label $y_{uv}=-1$ consists of non-adjacent node pairs, including sibling nodes sharing the same parent and nodes from different branches of the hierarchy. The negative pair separation is defined as \(
d^- = \inf_{(u,v)\in\mathcal D^-} d(u,v).
\) 

Given hyperbolic embeddings $\mathbf z_u,\mathbf z_v \in \mathbb H$, prediction is performed based on the hyperbolic distance $d_{\mathcal H}(\mathbf z_u,\mathbf z_v)$, where smaller distances correspond to higher probabilities of edge existence. 
To model this relationship, we employ the same Fermi--Dirac decoder as  previously used to compute connection probabilities. Specifically, the scoring function is defined as
\begin{equation}
s(u,v) = r - \frac{d(u,v)}{t},
\end{equation}
and the corresponding probability is given by
\begin{equation}
\mathrm{prob}(u,v) = \sigma\!\bigl(s(u,v)\bigr)
= \frac{1}{1 + \exp\!\bigl(-s(u,v)\bigr)} .
\end{equation}



\section{Geometric Crowding and Generalization}

In this section, we develop a geometric framework that connects embedding geometry to generalization guarantees for hierarchical data. We first demonstrate that fixed-dimensional Euclidean embeddings inevitably suffer from layer-wise crowding when representing tree-structured data. This phenomenon induces a fundamental stretch–squeeze trade-off between preserving tree distances and maintaining within-layer separation. By relating geometric separation to the functional margin of pairwise distance-based models, we further reveal how Euclidean crowding weakens margin-based generalization guarantees. In contrast, hyperbolic geometry preserves separation more effectively and thus admit stronger generalization guarantees for near-tree structured data.

\subsection{Euclidean Crowding in Hierarchical Representations}

We begin by analyzing the inherent limitations of Euclidean geometry when embedding tree-structured data. The key tension arises from the \emph{exponential growth} of nodes in a tree versus the \emph{polynomial volume growth} of Euclidean hyperballs.
As the tree depth increases, this mismatch forces representations from different branches to become increasingly crowded.

The following theorem~\ref{thm:crowding_tradeoff} formally characterizes this phenomenon by analyzing the layer-wise minimum separation among Euclidean embeddings of tree nodes.

\begin{theorem}[Layer-wise crowding and stretch--squeeze trade-off]
\label{thm:crowding_tradeoff}
Let $T$ be a rooted tree with branching factor $k$ and let
$f:T\to\mathbb R^d$ be an embedding with $f(\mathrm{root})=0$.
Let
\(
\delta_l :=
\min_{u,v\in\mathcal V_l}
\|f(u)-f(v)\|_2
\)
be the minimum Euclidean separation among node embeddings at depth $l$.
For each level $i$, define the stretch factor
\(
\alpha_i :=
\max_{(u,v)\in\mathcal E_i}
\frac{\|f(u)-f(v)\|_2}{d_i},
\)
and let
\(
A_l=\sum_{i=1}^l \alpha_i d_i.
\)
Then there exists a constant $C$  such that
\begin{equation}
    \delta_l
\le
C\,(\sum_{i=1}^l \alpha_i d_i)\,k^{-l/d}.
\label{crowding_theorem}
\end{equation}

Consequently:

(i) If the embedding approximately preserves radial tree distances
($A_l \approx T_l$), then
\begin{equation}
   \delta_l \le  C\,(\sum_{i=1}^l d_i)\,k^{-l/d} \le C  d_1 l k^{-l/d}.
   \label{crowding_theorem1}
\end{equation}

In particular,  when ignoring the polynomial factor $O(l)$, we obtain \( \delta_l = \tilde O(k^{-l/d}),\)
indicating exponential crowding of nodes. 

(ii) Conversely, maintaining a uniform layer-wise separation
$\delta_l\ge m>0$ requires
\begin{equation}
   A_l \ge \frac{m}{C}k^{l/d}, 
   \label{crowding_theorem2}
\end{equation}

which implies exponential radial stretching of tree edges.

Together, these results reveal a fundamental stretch--squeeze trade-off
in Euclidean embeddings of hierarchical structures.
1Furthermore, the exponent $1/d$ indicates that higher embedding dimensions $d$ can partially mitigate this crowding effect.
\end{theorem}

\paragraph{Proof sketch.}
The proof uses a geometric packing argument.

At depth $l$, place a Euclidean ball of radius $\delta_l/2$ around each
embedded node $f(v)$ for $v\in\mathcal V_l$. These balls are pairwise
disjoint by definition of $\delta_l$.

On the other hand, the cumulative stretch
$A_l=\sum_{i=1}^l\alpha_i d_i$ bounds the maximal radial extent of
layer-$l$ embeddings.
Therefore all balls lie inside a Euclidean ball of radius
$A_l+\delta_l/2$.

Comparing the volumes of the disjoint small balls with that of the
containing ball yields

\[
\delta_l \le C A_l |\mathcal V_l|^{-1/d}.
\]

Since $|\mathcal V_l|=k^l$, this gives
$\delta_l \le C A_l k^{-l/d}$.

If $A_l$ remains proportional to the intrinsic tree distance $T_l$,
the separation $\delta_l$ necessarily decays exponentially with depth,
producing layer-wise crowding.
Conversely, maintaining a constant separation requires $A_l$ to grow
exponentially, leading to radial stretch along tree paths.
The full proof is provided in Appendix~A.

\subsection{From Geometric Separation to Functional Margin}

Geometric separation is closely related to learning performance through its influence on the functional margin.
To formalize this relationship, we consider the distance-based prediction setting introduced in Section~3.

\begin{lemma}[Separation–margin relation]
\label{lem:sep_to_margin}
Consider a pairwise distance-based model with score
$
s(u,v)=\frac{r-d(u,v)}{t},
$ where $t>0$, and labels $y_{uv}\in\{+1,-1\}$.
Let $
d^+ := \sup_{(u,v)\in\mathcal D^+} d(u,v)
$ denote the maximum positive-pair distance, and define the geometric separation gap
$\Delta := d^- - d^+$. Then the maximal achievable pairwise functional margin satisfies
\begin{equation}
\sup_{r\in\mathbb R}\, \hat\gamma
\;=\;
\frac{\Delta}{2t}.
\end{equation}
\end{lemma}

This lemma establishes a quantitative relationship between geometric separation and the maximal achievable pairwise functional margin.
Reductions in the separation gap $\Delta$ are generally accompanied by a corresponding decrease in the maximal achievable margin, independent of model capacity or optimization details. Since margin-based generalization bounds typically scale inversely with the margin, geometric distortion directly leads to weaker generalization guarantees. Detailed proofs are provided in Appendix~A.2.

\subsection{Margin-Based Generalization Advantage of Hyperbolic Metrics}

Tree metrics admit embeddings into hyperbolic space with arbitrarily small distortion~\cite{distortion_lilun}, preserving geometric separation  even at large depths. In contrast, Theorem~4.1 shows that fixed-dimensional Euclidean embeddings inevitably suffer from layer-wise crowding when representing exponentially expanding hierarchical structures, which compresses the separation between distant nodes. Since the functional margin is determined by the geometric separation between positive and negative pairs (Lemma~\ref{lem:sep_to_margin}), this crowding effect leads to a reduced attainable margin in Euclidean space.
Combining  margin-based Rademacher complexity  bound~\cite{bartlett2002rademacher} with the separation--margin relation in Lemma~\ref{lem:sep_to_margin}, we obtain the following result.

\begin{theorem}[Hyperbolic margin-based generalization advantage]
\label{thm:hyperbolic_generalization}
Consider representation learning on near-tree structured data with a pairwise
score function $s(u,v)$. Let $\Delta = d^- - d^+$ denote the geometric separation
between negative and positive pairs, and let
\(
\hat\gamma=\frac{\Delta}{2t}
\)
be the maximal achievable functional margin (Lemma~\ref{lem:sep_to_margin}).

If the representations are learned in hyperbolic space $\mathbb H$ and
Euclidean space $\mathbb R^d$, respectively, then for tree-like structured data
\(
\Delta_{\mathbb H}>\Delta_{\mathbb R},
\)
which implies
\(
\gamma_{\mathbb H}>\gamma_{\mathbb R}.
\)
 in the margin-based generalization bound
\begin{equation}
\mathcal R_{0\text{-}1}(s)
\le
\widehat{\mathcal R}_{\gamma}(s)
+
\frac{2}{\gamma}\mathfrak R_n(\mathcal F)
+
\sqrt{\frac{\log(2/\delta)}{2n}}
\end{equation}
the hyperbolic model admits a smaller margin-dependent complexity term
\(
\frac{2}{\gamma_{\mathbb H}}\mathfrak R_n(\mathcal F)
<
\frac{2}{\gamma_{\mathbb R}}\mathfrak R_n(\mathcal F).
\)
Therefore, when empirical margin errors are comparable, models equipped
with hyperbolic metrics admit tighter margin-based generalization guarantees
than their Euclidean counterparts.
\end{theorem}

Theorem~\ref{thm:hyperbolic_generalization} provides a principled geometric
explanation for the empirical success of hyperbolic representations: 
by preserving larger separation between positive and negative pairs, hyperbolic geometry enables larger functional margins and consequently stronger generalization guarantees on hierarchical data. Detailed derivations are provided in
Appendix~A.3.

\section{From Theory to Model: The EH-Fusion Framework}

Motivated by the theoretical analysis in Section~\ref{thm:crowding_tradeoff}, we propose a lightweight framework that exploits the metric advantages of hyperbolic geometry while remaining the flexibility of Euclidean graph encoders. Unlike fully hyperbolic architectures that replace all Euclidean operations, our approach introduces a minimal geometric realization layer and an adaptive mechanism for balancing between Euclidean and hyperbolic geometries based on local structural patterns. The proposed EH-Fusion framework consists of two main components: (i) a multi-geometry embedding module that maps shared Euclidean representations onto different manifolds via lightweight exponential mapping layers, and (ii) a node-wise adaptive fusion module that dynamically balances predictions from multiple geometries according to local neighborhood expansion patterns.


\begin{figure*}
    \centering
    \includegraphics[width=1\linewidth]{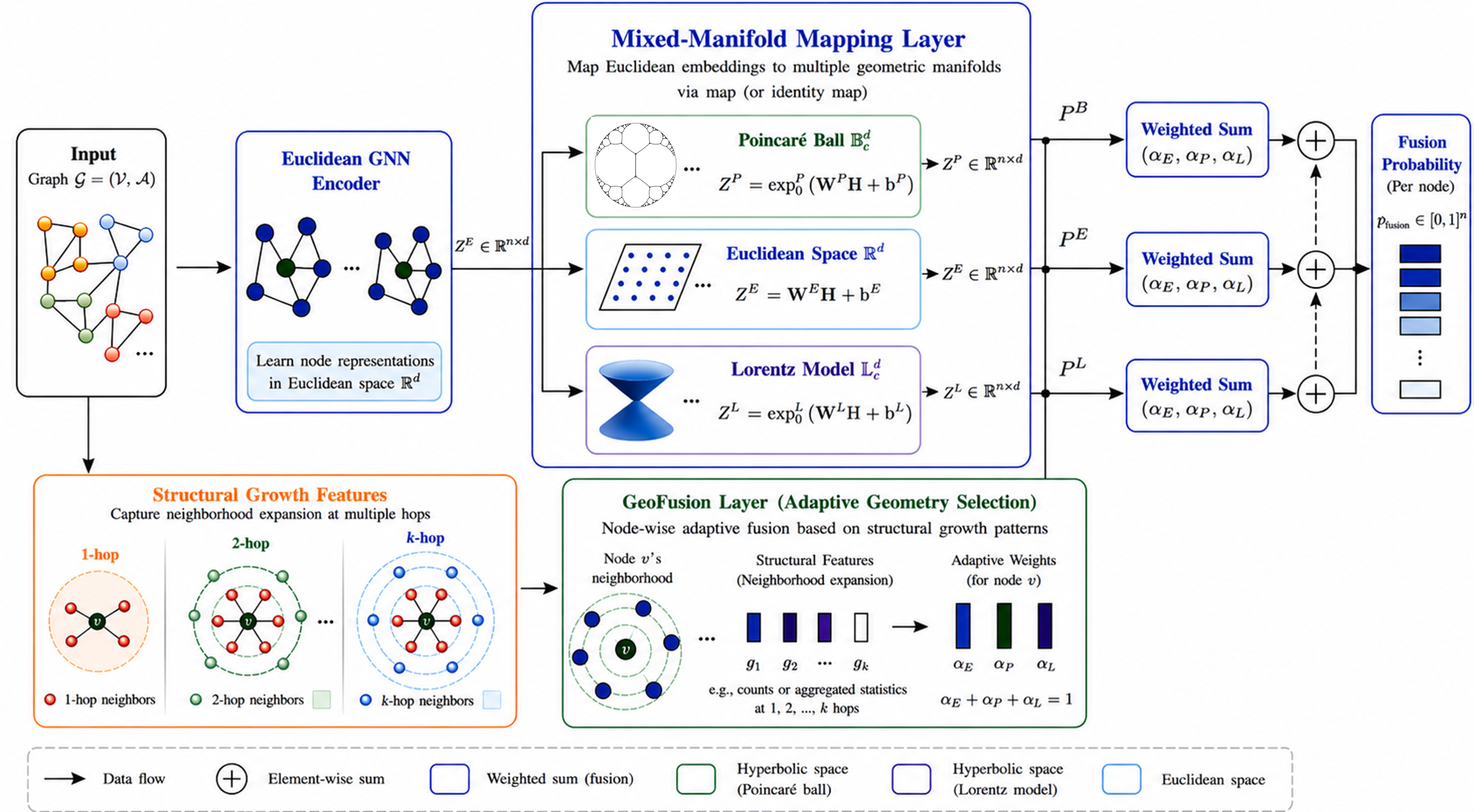}
    \caption{The overall framework of EH-Fusion: 
 \textbf{Upper part:} A shared Euclidean graph encoder yields node representations,  which are then mapped onto multiple geometric manifolds for metric learning.
\textbf{Lower part:} High-order degree growth information is extracted to capture local geometric characteristics and guide the adaptive aggregation of manifold-specific prediction probabilities. }
    \label{fig:framework}
\end{figure*}

Figure~\ref{fig:framework} provides an overview of the framework.
Starting from a shared Euclidean graph encoder,  
node representations are projected into three geometric spaces: Euclidean space $\mathbb{R}^d$, the Poincar\'e ball model $\mathbb{B}^d_c$, and the Lorentz model $\mathbb{L}^d_c$ of hyperbolic geometry. 
To handle graphs with heterogeneous geometric structures, we further introduce a topology-aware selection mechanism hat uses high-order degree statistics to compute node-specific mixing weights, enabling adaptive aggregation of predictions from different manifolds.

\subsection{Multi-Geometry Embedding via Lightweight Exponential Mapping layers}

Given a graph $G=(\mathcal V,\mathcal E)$ with node features $X$ and adjacency matrix $A$, we first obtain Euclidean node representations through a standard graph neural network:
\begin{equation}
\mathbf H = \mathrm{Encoder}_{\mathrm E}(X,A),
\qquad
\mathbf H \in \mathbb R^{|\mathcal V|\times d}.
\end{equation}
where the Euclidean embedding $\mathbf H$ serves as a shared latent representation and $\mathrm{Encoder}_{\mathrm{E}}(\cdot)$ can be any Euclidean GNN such as GCN or GAT. 
All message passing and feature aggregation are performed in Euclidean space,  allowing the model to leverage mature architectural designs and stable optimization schemes. From this share representation $\mathbf H$, we obtain embeddings in different geometric spaces via  geometry-consistent mappings.

\paragraph{\textbf{Euclidean Space Representation}}
The final Euclidean embeddings are obtained through a linear projection: 
\begin{equation}
    \mathbf Z^{\mathrm E} = \mathbf W^E  \mathbf H+ \mathbf b^E
\end{equation} 
 where $\mathbf W^{E}$ and $\mathbf b^{E}$ are learnable parameters. This projection is kept separate from the hyperbolic mappings to reduce interference during optimization.
 
\paragraph{\textbf{Poincar\'e Space Representation}}
To capture hierarchical relations, we map Euclidean representations to a Poincar\'e ball space via the exponential map at the origin:
\begin{equation}
\mathbf Z^{\mathrm P} = \exp^{\mathrm P}_{0}(\mathbf W^P \mathbf H+\mathbf b^P),
\end{equation}
where $\exp^{\mathrm P}_{0}(\cdot)$ denotes the exponential map of the Poincar\'e ball.
The Poinca\'e ball model provides an intuitive geometric interpretation
of hierarchical structures, where the radial coordinate roughly
corresponds to the hierarchy level, while angular separation reflects
similarity between sibling nodes.

\paragraph{\textbf{Lorentz Space Representation}}
In parallel, we also consider the Lorentz space model of hyperbolic geometry, which offers numerically stable distance computations and is convenient for large-scale graphs:
\begin{equation}
\mathbf Z^{\mathrm L} = \exp^{\mathrm L}_{0}(\mathbf W^L\mathbf H+\mathbf b ^L),
\end{equation}
where $\exp^{\mathrm L}_{0}(\cdot)$ is the exponential map defined on the Lorentz manifold.

The parameters $W^E, W^P, W^L$ and biases are learned independently for each geometry, enabling the model to learn geometry-specific representations.

\subsection{Node-wise Adaptive geometry Selection}

 Real-world graphs often exhibit mixed geometric structures: some nodes lie in regions with tree-like expansion patterns that are more naturally modeled by hyperbolic geometry, while others are better approximated by locally Euclidean geometry.
Enforcing a single global geometry across the entire graph is therefore suboptimal. To address this, we propose a data-driven fusion mechanism that adaptively selects the most suitable geometry for each node based on its local neighborhood expansion pattern.

The underlying principle is rooted in the volume growth properties of the two geometries. Euclidean space exhibits polynomial volume growth, whereas hyperbolic space grows exponentially. This distinction manifests in graph-structured data as differences in how a node's neighborhood expands with distance. Consequently, the pattern of local expansion serves as a structural signature that reveals which geometry better aligns with the node's latent structure.

\paragraph{\textbf{Extracting structural growth features}}
To capture this expansion behavior, we extract node-wise structural signals, including i-hop neighborhood degrees:
\begin{equation}
\mathbf S
=
\bigl[
\operatorname{diag}(A),
\operatorname{diag}(A^2),
\ldots,
\operatorname{diag}(A^k)
\bigr],
\end{equation}
where $\operatorname{diag}(A^i)$ denotes the number of length-$i$ closed walks rooted at given nodes.
Rapid growth in high-order degrees tend to exhibit tree-like or hierarchical structures, whereas slower growth is more consistent with locally Euclidean geometry.
\paragraph{\textbf{Learning node-wise geometry mixing weights}}
Based on the structural growth features, we introduce a lightweight attention mechanism that assigns each node a set of geometry-specific weights.
For node $v$,
its feature vector $\mathbf{S}_v$ is passed through a small multilayer perceptron $\phi(\cdot)$ followed by a softmax normalization:
\begin{equation}
\boldsymbol{\alpha}_v
=
\mathrm{Softmax}\!\left(
\phi(\mathbf S_v)
\right),
\qquad
\boldsymbol{\alpha}_v \in \mathbb R^{M},
\end{equation}
where $M = 3$ corresponds to the three geometric models: Euclidean ($\mathbb{R}^d$), Poincaré ball ($\mathbb{B}^d_c$), and Lorentz ($\mathbb{L}^d_c$). 
 $\phi(\cdot)$ is implemented as a multilayer perceptron with two hidden layers.
\paragraph{\textbf{Adaptive fusion of multi-geometry predictions}}
The final link probability is obtained as a convex combination of the geometry-specific probabilities: 
\begin{equation}
\mathbf z_v
=
\alpha_v^{\mathrm E}\,p_v^{\mathrm E}
+
\alpha_v^{\mathrm P}\, p_v^{\mathrm P}
+
\alpha_v^{\mathrm L}\, p_v^{\mathrm L}.
\end{equation}
This formulation allows each node to emphasize the geometry that best matches its local  growth pattern. Nodes in rapidly expanding, tree-like regions naturally assign higher weights to hyperbolic components, while nodes in flatter regions rely more on Euclidean representations.

\section{Experiment}

\subsection{Experimental setup}
 
\subsubsection{{Datasets}} 
We evaluate the proposed method on five graph learning benchmarks spanning citation networks, transportation networks, disease propagation graphs, and large-scale product co-purchase graphs. These datasets exhibit varying degrees of tree-likeness, which we quantify using Gromov $\delta$-hyperbolicity: smaller $\delta$ values indicate stronger hyperbolic (tree-like) characteristics.



\paragraph{\textbf{Cora\cite{cora}}}
Cora is a standard citation network dataset, where nodes represent scientific publications and edges denote citation relationships. Each node is associated with a high-dimensional bag-of-words feature vector.

\paragraph{\textbf{Citeseer\cite{citeseer}}}  The Citeseer dataset is another classic citation network with similar characteristics to Cora. Nodes represent scientific publications belonging to a set of distinct classes, and edges correspond to citation relationships. 

\paragraph{\textbf{Airport\cite{hgcn}}}  The Airport dataset models an air transportation network, where nodes correspond to airports and edges represent flight routes between airports. This graph exhibits real-world topology and structural diversity.

\paragraph{\textbf{Disease\cite{hgcn}}}  The Disease dataset represents a disease propagation network, where nodes denote individuals or entities and edges reflect potential transmission paths of a disease process. 


\paragraph{\textbf{Ele-Fashion\cite{mmgraph}}} Ele-Fashion is a node classification dataset constructed from Amazon-Fashion, where each node corresponds to a product image and edges indicate whether two products are co-purchased. The graph contains 97,766 nodes and 199,602 edges.


Together, these datasets provide a diverse evaluation  spanning different scales, domains, and levels of hyperbolicity, enabling a comprehensive assessment of the proposed framework.


\subsubsection{{Baselines}}

To comprehensively evaluate the effectiveness of our method, we compare against representative graph  learning approaches spanning both Euclidean and hyperbolic geometries.

\paragraph{\textbf{Poincaré Embeddings (PoincaréEmb)\cite{NIPS2017_poincare}}} Poincaré Embeddings learn node representations directly in hyperbolic space using the Poincaré ball model.  It has been widely adopted as a foundational hyperbolic embedding baseline.


\paragraph{\textbf{Graph Convolutional Network (GCN)\cite{gcn}}} GCN performs neighborhood aggregation through normalized adjacency-based propagation, followed by linear feature transformation. It serves as a Euclidean graph learning baseline.

\paragraph{\textbf{Graph Attention Network (GAT)\cite{gat}}}  GAT incorporates attention mechanisms into neighborhood aggregation, allowing nodes to adaptively weight contributions from different neighbors.
 

\paragraph{\textbf{Hyperbolic Neural Networks (HNN)~\cite{hgnn}}}
HNN extends classical neural networks to hyperbolic space by reformulating basic neural operations, including linear transformations, bias addition, nonlinear activation, and optimization, under Riemannian geometry. 

 
\paragraph{\textbf{Hyperbolic Graph Convolutional Networks (HGCN)~\cite{hgcn}}}
HGCN generalizes graph convolution to hyperbolic space by performing feature transformation and neighborhood aggregation through tangent-space approximation.
 


\paragraph{\textbf{Fully Hyperbolic Neural Network(FHNN)\cite{fhnn}}} FHNN is a fully hyperbolic graph neural network in which both feature transformations and aggregation operations are conducted entirely in lorentz hyperbolic space. 




\subsubsection{Implementation Details}

We tune the learning rate in $\{0.01, 0.008, 0.005, 0.001\}$, 
the L2 regularization coefficient in $\{0, 10^{-1}, 5\times10^{-2}, 10^{-2}, 5\times10^{-3}, 10^{-3}, 5\times10^{-4}, 10^{-4}\}$, 
and the dropout rate from $0.0$ to $0.7$ with a step size of $0.1$. 
All hyperparameters are selected based on validation performance.
For the image-attributed graph EF-Fashion, where each node is associated with visual features, we perform a grid search over 21 learning rates, ranging from $0.0001$ to $0.0009$, $0.001$ to $0.009$, and $0.01$ to $0.03$, and report the best-performing results. To fully exploit the visual modality, we evaluate both multimodal-aligned encoders (e.g., CLIP and ImageBind) and general-purpose visual encoders (e.g., DINO and ViT).
The latent dimension  $d$is set to $ \{8,16,32,64\}$ for all methods. Poincar\'e embeddings are optimized using Riemannian stochastic gradient descent (RSGD)~\cite{rsgd}, while GCN and GAT-based models are optimized using Adam~\cite{adam}. 
All GNN-based models adopt early stopping, where training is terminated if the validation performance does not improve for 100 consecutive epochs.
All experiments are conducted on an Intel Xeon Silver 4410Y CPU (2.0--3.9\,GHz) and an NVIDIA GeForce RTX 4090Ti GPU.

\begin{table*}[ht]
\centering
\small
\setlength{\tabcolsep}{2pt}
\renewcommand{\arraystretch}{1.2}
\caption{AUC(\%) for link prediction task.}
\begin{tabular}{
c
cccc @{\hspace{12pt}}
cccc @{\hspace{12pt}}
cccc @{\hspace{12pt}}
cccc
}
\toprule
\multirow{2}{*}{Model}
& \multicolumn{4}{c}{Disease ($\delta = 0$)}
& \multicolumn{4}{c}{Airport ($\delta = 1$)}
& \multicolumn{4}{c}{Citeseer ($\delta = 5$)}
& \multicolumn{4}{c}{Cora ($\delta = 11$)} \\

\cmidrule(lr){2-5}
\cmidrule(lr){6-9}
\cmidrule(lr){10-13}
\cmidrule(lr){14-17}

& 8 & 16 & 32 & 64
& 8 & 16 & 32 & 64
& 8 & 16 & 32 & 64
& 8 & 16 & 32 & 64 \\

\midrule
poinEmb
& 70.5 & 71.8 & 75.2 & 78.2
& 96.2 & 96.5 & 96.7 & 96.7
& 81.4 & 80.7 & 80.3 & 80.9
& 86.1 & 86.4 & 87.8 & 87.8 \\

HNN
& 94.2 & 94.0& 94.2 & 94.1
& 93.8 & 94.1 & 94.3 & 94.3
& 95.5 & \underline{96.2} & 96.1 & 96.3
& 92.6 & 92.8 & 93.7 & 93.5 \\

\midrule
GCN
& 85.8 & 82.4 & 85.5 & 87.1
& 93.1 & 94.1 & 93.9 & 95.1
& 92.2 & 93.7 & 95.2 & 95.8
& 90.6 & 92.6 & 92.9 & 93.8
\\

GAT
& 81.9 & 82.8 & 82.2 & 89.9
& 94.3 & 93.3 & 95.3 & 94.1
& 94.8 & 95.8 & \underline{96.6} & \textbf{97.2}
& 92.3 & 93.8 & 94.5 & \underline{95.2}\\


\midrule
HGCN
& 89.3 & 89.3 & 89.9 & 89.8
& 94.4 & 95.0 & 95.3 & 96.2
& 95.4 & 96.1 & 96.4 & 96.6
& \underline{93.8}& 94.6 & \underline{95.0}& 95.0
 \\

FHNN
& 96.7& 96.5 & 96.8 & 96.0
& \underline{97.3} & 96.3 & 96.7 & 96.7
& 91.3 & 90.8 & 91.0 & 95.0
& 92.5 & \underline{94.7}& \underline{95.0}& 94.9\\

\midrule
GCN2P(ours)
& 94.9 & 95.1 & 95.0 & 94.6
& 95.9 & \underline{96.7}& 96.8 & 96.9
& 93.3 & 93.4 & 93.8 & 93.9
& 92.0 & 92.4 & 93.1 & 93.9
 \\

GCN2L(ours)
& 94.6 & 95.7 & 95.4 & 94.9
& 96.2 & 95.6 & \underline{97.1}& 97.2
& 93.4 & 93.9 & 93.2 & 92.0
& 91.9 & 92.2 & 92.3 & 92.4
 \\

\midrule
GAT2P(ours)
& 96.9 & 97.1& 96.5 & 97.4
& 95.4 & 95.8 & 96.0 & 96.5
& 93.1 & 94.6 & 95.0 & 95.2
& 92.9 & 94.0 & 94.3 & 94.3
 \\

GAT2L(ours)
& \textbf{97.7} & 97.6 & \underline{97.3}& \underline{97.5}
& 95.9 & 95.8 & 96.3 & 96.1
& 93.1 & 93.9 & 94.5 & 95.3
& 92.7 & 93.7 & 94.1 & 94.8
 \\

\midrule
GCN-EH-Fusion(ours)
& \underline{97.3}& \textbf{98.4} & \textbf{98.0}& 95.4
& 96.8 & \textbf{97.3}& \textbf{97.2} & \underline{97.3}
& \underline{95.7} & 95.6 & 95.7 & 95.9
& 93.4& 93.8& 93.5 & 94.1
 \\

GAT-EH-Fusion(ours)
& 97.1 & \underline{98.3} & 97.2 & \textbf{98.1}
& \textbf{97.7} & \textbf{97.3} & \textbf{97.2} & \textbf{97.7}
& \textbf{96.0} & \textbf{96.5} & \textbf{96.9} & \underline{96.9}
& \textbf{94.3} & \textbf{95.0} & \textbf{95.4} & \textbf{95.3}
 \\

\bottomrule
\end{tabular}

\end{table*}

\subsection{Experimental Results}

Table~1 summarizes the performance comparison across four benchmark datasets with varying $\delta$-hyperbolicity levels, where smaller $\delta$ indicates stronger tree-likeness and more pronounced hyperbolic structure. The empirical results reveal a clear relationship between dataset hyperbolicity and geometric modeling effectiveness. On strongly hyperbolic datasets such as Disease ($\delta=0$) and Airport ($\delta=1$), introducing hyperbolic realization leads to substantial and consistent improvements over purely Euclidean baselines. 

For example,  On strongly hyperbolic datasets such as Disease ($\delta=0$), hyperbolic modeling brings substantial improvements over purely Euclidean baselines. 
For instance, Euclidean GCN achieves at most 87.1\%, while Euclidean GAT reaches 89.9\%. 
In contrast, hyperbolic variants such as  FHNN achieve up to 96.0\%, yielding nearly a 6
\% points improvement over Euclidean GAT. 
Our lightweight hyperbolic extensions, GAT2P and GAT2L, further improve the performance to 97.4\% and 97.5\%, respectively. 
The proposed GAT-EH-Fusion achieves up to 98.1\%, surpassing the strongest single-geometry hyperbolic baseline by approximately 1.3\%.
A similar trend is observed on the Airport dataset ($\delta=1$). 
Euclidean GCN achieves 95.1\%, and Euclidean GAT reaches 94.1\%. 
Hyperbolic baselines such as HGCN and FHNN improve performance to 96.2\% and 96.7\%, respectively. Our GAT2L achieves 96.1\%, while GAT2P reaches 96.5\%, again matching most  hyperbolic baselines. The adaptive GAT-EH-Fusion further improves the result to 97.7\%, achieving the best overall performance. 

Notably, GAT2P and GAT2L introduce only a single hyperbolic exponential mapping layer on top of the original Euclidean GAT backbone, without redesigning the entire architecture into a fully hyperbolic network. 
Despite this minimal modification, they already achieve performance comparable to fully hyperbolic models such as HGCN and FHNN. 
This observation is consistent with the theoretical analysis in Section~4, which suggests that the primary advantages of hyperbolic modeling arise from the metric structure of the representation space.
In particular, once node representations are realized in a geometry whose volume growth matches the exponential expansion of tree-like graphs, hierarchical separation can be better preserved under the hyperbolic distance. 
Therefore, even a lightweight hyperbolic realization is sufficient to effectively capture the structural characteristics of strongly hierarchical graphs.
As the $\delta$ value increases (e.g., Cora with $\delta=5$ and Citeseer with $\delta=11$), the relative advantage of hyperbolic modeling becomes less pronounced. On Cora, Euclidean GAT achieves 95.2\%, which is already comparable to hyperbolic baselines around 95.0\%. Similarly, on Citeseer, Euclidean GAT reaches 97.2\%, remaining competitive with hyperbolic counterparts. In contrast to the strongly hyperbolic datasets, introducing only a single hyperbolic exponential mapping layer (e.g., GAT2P and GAT2L) brings limited  improvement in these cases, as the underlying graph structure exhibits weaker hierarchical expansion and is more compatible with Euclidean geometry. This suggests that when the graph is moderately or weakly hyperbolic, Euclidean geometry already provides an effective inductive bias, and fixed hyperbolic realization may yield  slight performance fluctuations.

\begin{table}[h]
\caption{ACC(\%) performance comparison for image-attributed graph.}
\centering
\setlength{\tabcolsep}{4pt}
\begin{tabular}{l|l|cccc}
\toprule
 & & \multicolumn{4}{c}{Ele-Fashion} \\
Model & Image Encoder & dim=8 & dim=16 & dim=32 & dim=64 \\
\hline
\multirow{4}{*}{GCN}
& ImageBind & 79.45 & 82.43 & 82.46 & 83.52 \\
& CLIP      & 78.73 & 79.53 &\textbf{ 80.91} & 80.57 \\
& DINOv2    & \textbf{79.50} & 80.13 & 80.07 & 81.84 \\
& ViT       & 78.35 & 80.46 & 79.80 & 80.81 \\
\hline
\multirow{4}{*}{GCN2L}
& ImageBind & \textbf{80.34} & \textbf{83.08} & \textbf{83.30} & \textbf{84.29} \\
& CLIP      & \textbf{80.19} & \textbf{80.90} & 80.52 & \textbf{81.59}\\
& DINOv2    & 78.77 & \textbf{81.37} &\textbf{81.38} &\textbf{82.26} \\
& ViT       & \textbf{80.25} & \textbf{80.93} & \textbf{81.24} &\textbf{81.31} \\
\bottomrule
\end{tabular}
\end{table}
We observe that  the proposed EH-Fusion models maintain stable performance across all datasets , achieving 95.3\% on Cora and 96.9\% on Citeseer. 
Importantly, this robustness is not limited to high-dimensional settings. 
Even under low embedding dimensions (e.g., 8, 16, and 32), EH-Fusion consistently achieves competitive or superior performance, demonstrating strong representation efficiency. This behavior highlights the benefit of adaptive manifold selection in handling structural heterogeneity while maintaining stable performance even in low-dimensional regimes.  By dynamically balancing Euclidean and hyperbolic representations  according to dataset geometry, EH-Fusion achieves robust and consistently strong performance across varying levels of hyperbolicity and embedding dimensions.

Table~2 reports the node classification results on the Ele-Fashion dataset under different visual encoders and embedding dimensions. We observe that the hyperbolic variant (GCN2L) generally outperforms the Euclidean GCN across different visual encoders and embedding dimensions, with consistent improvements observed in the majority of settings. These improvements are also relatively stable across different configurations. 
For example, with ImageBind features, GCN2L improves the performance from $79.45\%$ to $80.34\%$ at dim=8, from $82.43\%$ to $83.08\%$ at dim=16, and from $82.46\%$ to $83.30\%$ at dim=32. 
Similar trends can be observed for other encoders such as CLIP, DINOv2, and ViT.
Overall, this pattern suggests that the performance gains brought by hyperbolic modeling are largely independent of the specific visual encoder. 
In addition, the relative advantage of GCN2L remains stable across different dimensionalities, further highlighting the robustness of hyperbolic modeling.

\subsection{analysis}

\subsubsection{Ablation Study }


\begin{table}[ht]
\centering
\caption{Ablation study of the hyperbolic mapping layer and EH-Fusion module.}
\begin{tabular}{lcccc}
\toprule
 & \makecell{\textbf{Disease} } 
 & \makecell{\textbf{Airport} } 
 & \makecell{\textbf{Citeseer} } 
 & \makecell{\textbf{Cora}} \\
\midrule
\textbf{HGCN} & 89.9 & 96.2 & 96.6 & 95.0 \\
\textbf{GCN} & 87.1 & 95.1 & 95.8 & 93.8 \\
\textbf{+Exp2P} & 94.6 & 96.9 & 93.9 & 93.9 \\
\textbf{+Exp2L} & 94.9 & 97.2 & 92.0 & 92.4 \\
\textbf{+Exp+EH-Fusion} & 95.4 & 97.3 & 95.9 & 94.1 \\
\midrule
\textbf{FHNN} & 96.0 & 96.7 & 95.0 & 94.9 \\
\textbf{GAT} & 89.9 & 94.1 & \textbf{97.2} & 95.2 \\
\textbf{+Exp2P} & 97.4 & 96.5 & 95.2 & 94.3 \\
\textbf{+Exp2L} & 97.5 & 96.3 & 95.2 & 94.8 \\
\textbf{+Exp+EH-Fusion} & \textbf{98.1} & \textbf{97.7} & 96.9 & \textbf{95.3} \\
\bottomrule
\end{tabular}
\end{table}

\begin{figure}
    \centering
    \includegraphics[width=1\linewidth]{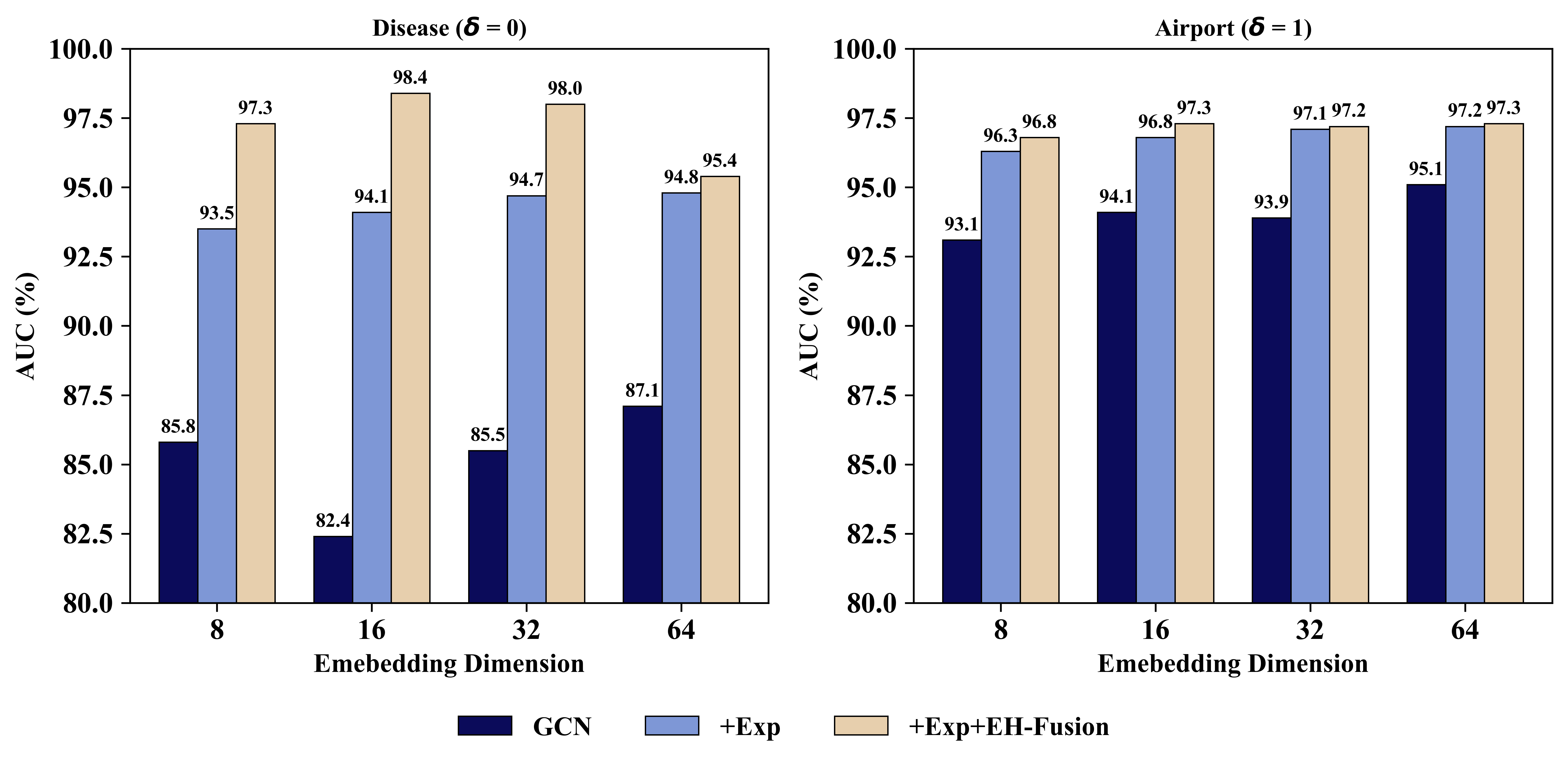}
    \caption{Effect of the exponential mapping layer and EH-Fusion on strongly hyperbolic datasets.  A single hyperbolic mapping significantly improves performance, and adaptive fusion provides further gains across embedding dimensions. }
    \label{fig:ablation}
\end{figure}

To evaluate the contribution of each component, we conduct ablation experiments by progressively adding the hyperbolic exponential mapping layer and the adaptive EH-Fusion mechanism to standard Euclidean backbones (GCN and GAT). Table~2 summarizes the results.

\paragraph{Effect of Single Hyperbolic Mapping.}
Adding a single exponential mapping layer to project Euclidean representations into either the Poincar\'e or Lorentz model consistently improves performance on strongly hyperbolic datasets. On Disease with $\delta=0$, GCN improves from 87.1\% to 94.6\% under the Poincar\'e mapping and to 94.9\% under the Lorentz mapping, yielding gains of approximately 7.5 to 7.8 percentage points.
Similarly, GAT increases from 79.7\% to 97.4\% with the Poincar\'e mapping and to 97.5\% with the Lorentz mapping. On Airport with $\delta=1$, GCN rises from 95.1\% to 97.2\% under the Lorentz mapping, while GAT improves from 94.1\% to 96.5\% under the Poincar\'e mapping. These improvements indicate that even a lightweight hyperbolic realization can substantially enhance modeling capacity when the graph exhibits strong hierarchical structure.

As the $\delta$ value increases, the gains from fixed hyperbolic mapping become smaller and occasionally fluctuate. On Cora with $\delta=5$, GCN increases marginally from 93.8\% to 93.9\% under the Poincar\'e mapping but decreases to 92.4\% under the Lorentz mapping. On Citeseer with $\delta=11$, GCN drops from 95.8\% to 92.0\% under the Lorentz mapping. These observations suggest that when the dataset structure is less hierarchical, imposing a fixed hyperbolic geometry may not consistently improve performance. In addition,we show introducing the Fusion module enhances robustness across datasets.
On Disease, GCN with EH-Fusion achieves 95.4\%, and GAT with EH-Fusion reaches 98.1\%, surpassing their single-geometry counterparts.
On Airport, GAT with EH-Fusion achieves 97.7\%. 
 Importantly, on datasets with larger $\delta$ values, EH-Fusion mitigates the degradation observed under fixed hyperbolic mappings. For example, on Citeseer, GCN with EH-Fusion recovers performance to 95.9\%, compared with 92.0\% under the Lorentz mapping. Similarly, GAT with EH-Fusion maintains competitive performance at 96.9\%.

Overall, the ablation results indicate that  a single hyperbolic realization layer already provides substantial gains on strongly hierarchical graphs, and the adaptive manifold fusion further enhances stability and robustness across datasets with heterogeneous structural properties.




\begin{figure}[ht]
    \centering
    \includegraphics[width=0.8\linewidth]
    {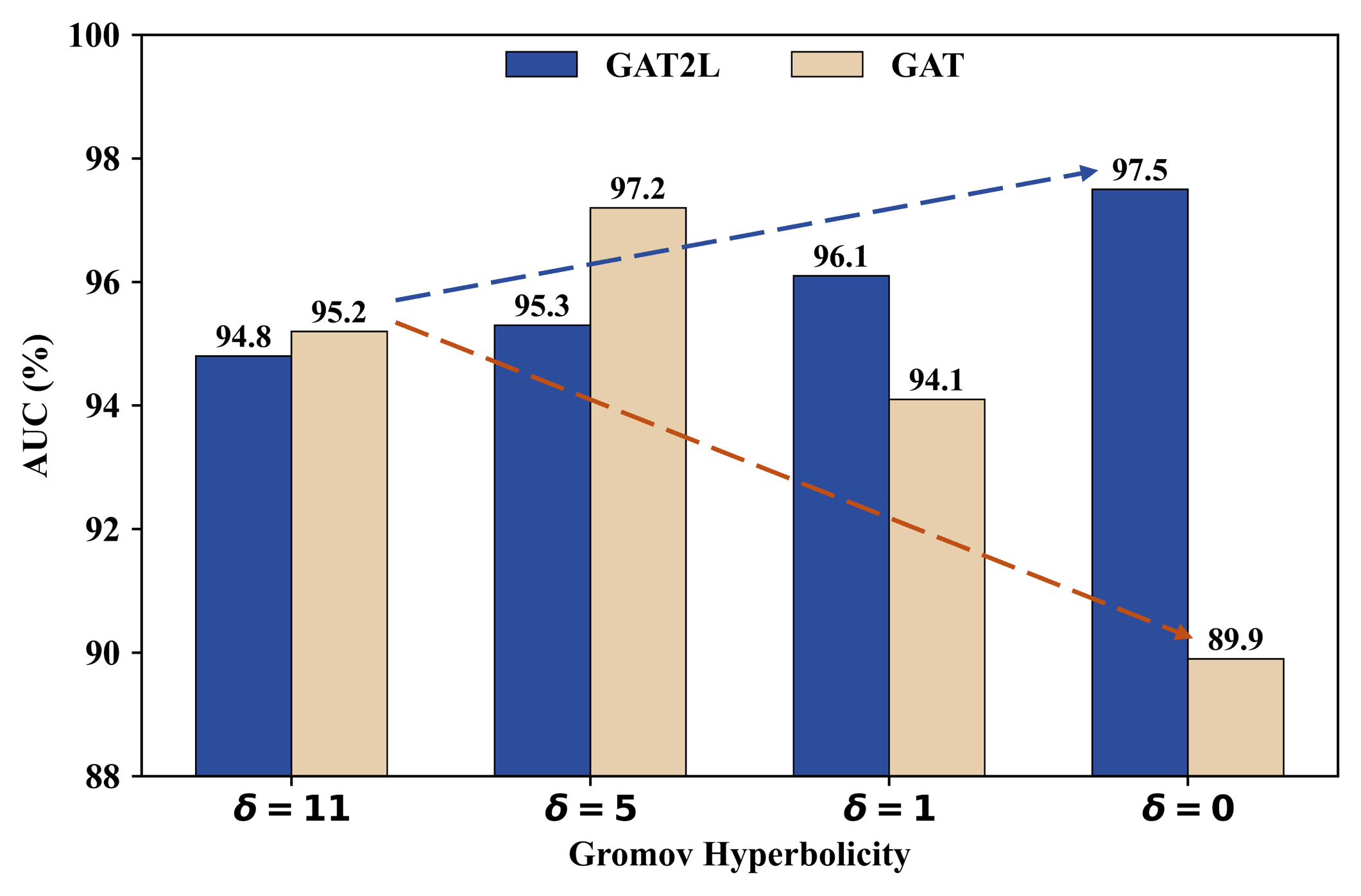}
    \caption{AUC comparison of GAT in Lorentz hyperbolic space and Euclidean space on datasets with varying $\delta$-hyperbolicity levels. Smaller $\delta$ indicates stronger hyperbolicity. }
    \label{fig:delta_gat_64}
\end{figure}


\subsubsection{Effect of Gromov Hyperbolicity}

Figure~\ref{fig:delta_gat_64} compares the performance of GAT in hyperbolic and Euclidean spaces under different Gromov $\delta$-hyperbolicity levels.
From the perspective of increasing hyperbolicity (i.e., decreasing $\delta$), 
the performance of GAT in the Lorentz hyperbolic space(GAT2L) improves consistently. 
As $\delta$ decreases, the AUC of the hyperbolic model steadily improves from 94.8\% to 97.5\%, demonstrating its increasing advantage on strongly hierarchical data. In contrast, Euclidean GAT exhibits an overall opposite trend: after a slight improvement at moderate $\delta$, its performance deteriorates sharply from 97.2\% to 89.9 \% when the hierarchy becomes stronger.
When viewed from the direction of decreasing hyperbolicity  (i.e., increasing $\delta$), Euclidean GAT shows steadily improved performance, while the advantage of hyperbolic embeddings correspondingly diminishes and may even vanish.
This observation is consistent with our theoretical analysis. 
As the data becomes less hierarchical, neighborhood expansion departs from exponential growth, which alleviates the crowding effect and makes Euclidean geometry increasingly compatible with the underlying graph structure.

\subsubsection{Impact of EH-Fusion for Avoiding Geometric Mismatch}

\begin{figure}[ht]
    \centering
    \includegraphics[width=0.8\linewidth]{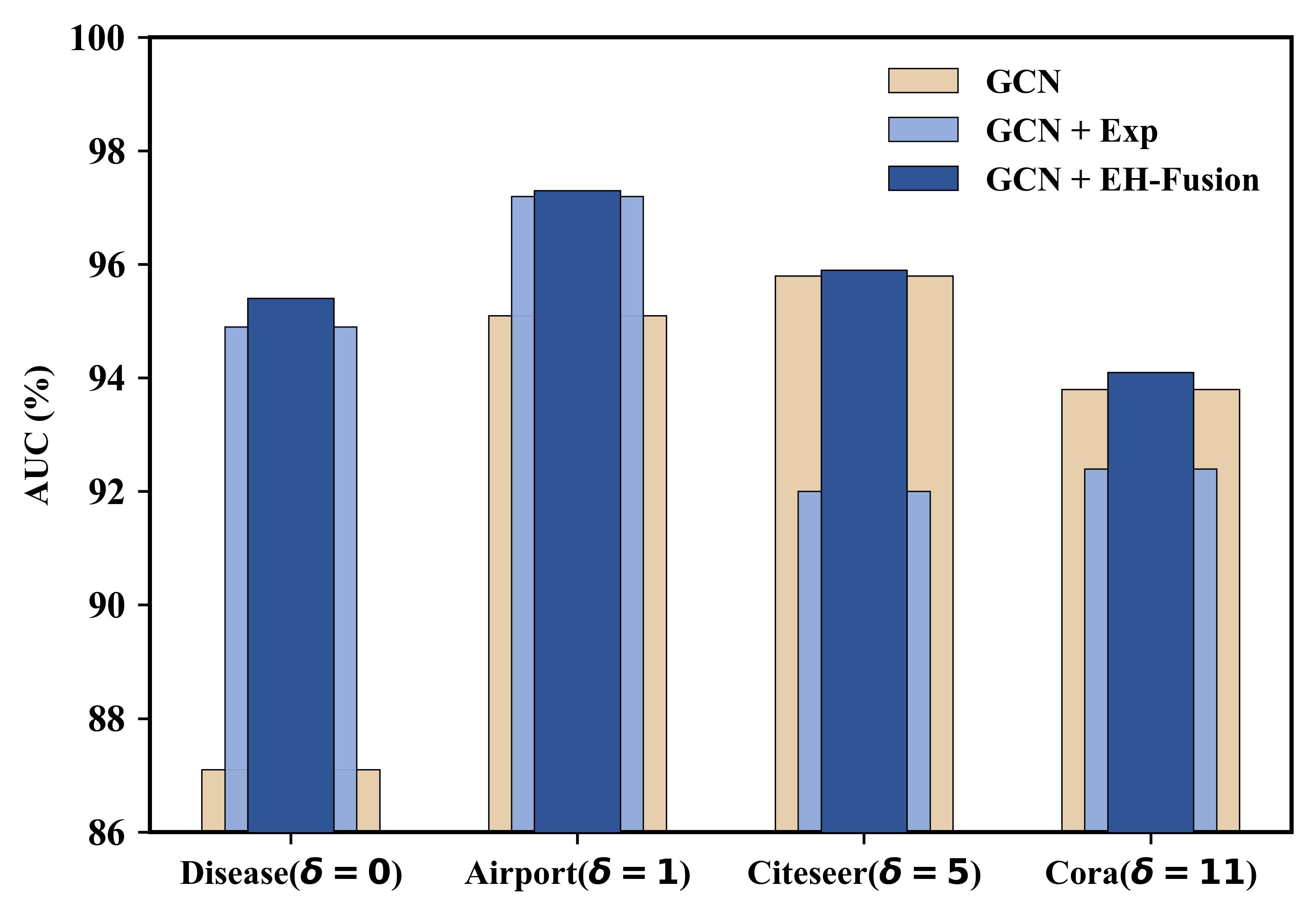}
    \caption{EH-Fusion analysis. Adaptive manifold selection mitigates geometric mismatch and stabilizes performance across datasets.}
    \label{fig:geometry-adaptation}
\end{figure}

Figure~4 compares GCN, GCN with a fixed hyperbolic exponential mapping, and GCN with the proposed EH-Fusion mechanism across datasets with different $\delta$-hyperbolicity levels.
The results reveal that fixed-geometry models may suffer from geometric mismatch when the assumed manifold does not align with the intrinsic structure of the dataset. On strongly hyperbolic data such as Disease with $\delta=0$, introducing a hyperbolic mapping significantly improves performance over the Euclidean baseline. However, as $\delta$ increases and the dataset becomes less tree-like, the limitation of fixed hyperbolic modeling becomes evident.
On Cora and Citeseer, the performance gain from a fixed hyperbolic mapping diminishes and, in some cases, slightly decreases compared to the Euclidean baseline. This indicates that when the graph structure is closer to polynomial growth, imposing a hyperbolic metric may introduce unnecessary geometric distortion.

In contrast, EH-Fusion consistently maintains or improves performance across all datasets. By dynamically balancing Euclidean and hyperbolic representations according to structural signals, the model avoids committing to a single manifold assumption. When the data exhibits strong hierarchy, the hyperbolic component is emphasized. When the structure is flatter, Euclidean representations dominate. This adaptive mechanism effectively mitigates geometric mismatch and provides robust performance under varying levels of hyperbolicity.

\subsubsection{Impact of Backbone Model}

Figure~5 illustrates the performance comparison between the Euclidean backbones GCN and GAT. Overall, stronger backbone models tend to achieve better performance, both equipped with a single exponential mapping layer or the EH-Fusion module. In particular, when the dataset exhibits strong hierarchical structure (i.e., small $\delta$-hyperbolicity), even a single hyperbolic realization is sufficient to capture the underlying hierarchical relationships.
\begin{figure}[ht]
    \centering
    \includegraphics[width=1\linewidth]{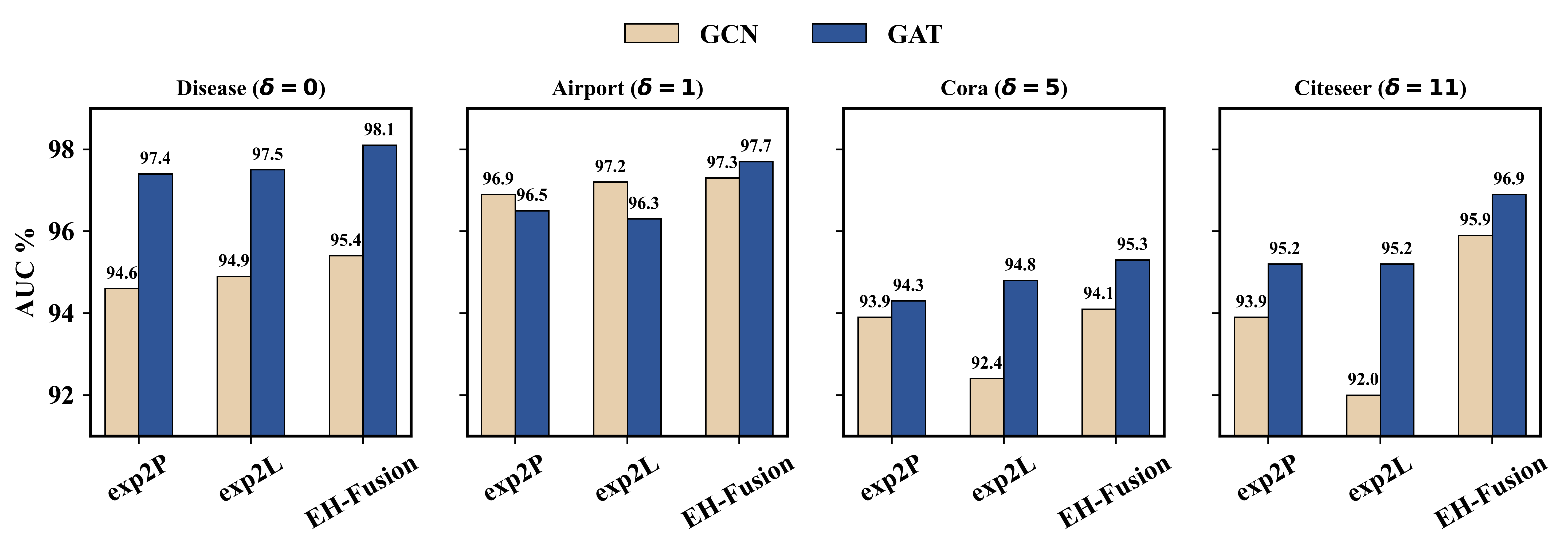}
    \caption{Effect of bacebone model  on Exponential Mapping and EH-Fusion moudule. }
    \label{fig:backbone-effect}
\end{figure}
From an optimization perspective,  the metric-based objective encourages hierarchical separation in the representation space, and this structural inductive bias is propagated backward via standard gradient-based optimization.  Consequently, when the backbone model possesses sufficient expressive capacity, introducing a single exponential mapping layer is often sufficient to exploit the hyperbolic metric and induce high-quality hierarchical representations. 
\textbf{This provides an explanation for the different design paradigms in hyperbolic representation learning.} In deep neural networks and large pretrained models, strong representational capacity often enables the backbone to produce high-quality latent features, making a single geometric realization layer sufficient to induce meaningful hierarchical representation under a hyperbolic metric. 
In contrast, graph neural networks are typically shallow to mitigate over-smoothing, which may limit  their expressive power.  Under these conditions, incorporating additional hyperbolic operations within the network can improve alignment with the target geometry and yield further performance gains.

Overall, the experimental results indicate that the effectiveness of Euclidean-to-hyperbolic (E2H) modeling is jointly influenced by the intrinsic latent geometry of the data and the expressive capacity of the baseline encoder.  Stronger backbones amplify the benefits of metric-based realization, while adaptive fusion mitigates geometric mismatch and ensures consistent performance under heterogeneous hyperbolicity levels.

\subsection{Discussion and Future Work}

This work provides a geometric and generalization-theoretic perspective on the advantage of hyperbolic representations for hierarchical data. By quantifying Euclidean crowding through layer-wise minimum separation and linking it to margin-based generalization, we explain why Euclidean embeddings struggle with tree-like structures. Furthermore, to handle graphs with mixed geometric structures, we propose EH-Fusion, a lightweight geometry-adaptive framework that balances Euclidean and hyperbolic metrics to leverage their complementary strengths, enabling improved geometric consistency.

Empirically, we observe that mixed-manifold models consistently outperform their single-manifold counterparts. Moreover, the relative advantage of each geometry is strongly governed by the intrinsic structure of the data: hyperbolic models are more effective for highly tree-like graphs with strong hyperbolicity, whereas Euclidean models become increasingly competitive as the degree of hyperbolicity decreases. These findings reinforce the connection between structural separability, geometric crowding, and representation effectiveness. Nevertheless, identifying a principled $\delta$-hyperbolicity threshold that can reliably guide the pre-selection of embedding geometry remains an open and practically important direction.

Beyond $\delta$-hyperbolicity, future work may explore more directly computable separability measures. To account for scale variations, one may consider a normalized variant of the separation margin, $\tilde{\Delta} = 1 - d^+/d^-$, which provides a scale-invariant characterization of structural separability. For computational convenience, crowding in hierarchical graphs can be characterized by the ratio between the average sibling distance and the average parent--child distance, which captures intra-layer compression. This idea naturally extends to general graphs by replacing sibling nodes with $k$-hop intra-layer neighbors and parent nodes with predecessor nodes from the previous layer.

These geometry-aware measures not only provide analytical characterization of learned representations but also offer practical utility for downstream tasks. Among various applications, we identify graph generation as the most promising direction. Specifically, separation-based geometric measures such as \(\Delta\) and its induced crowding metrics can serve multiple roles in graph generation: (1) as \textbf{evaluation criteria} for geometric fidelity, by comparing structural separability between real and generated graphs; (2) as \textbf{regularization terms} to suppress overly crowded structures and enable controllable generation of hierarchical patterns; and (3) as \textbf{conditioning signals} for the generation process, where a generator is trained to produce graphs whose \(\Delta\) falls within a specified range, allowing fine-grained control over branching patterns and hierarchical depth.

Overall, this perspective connects structural separability and local crowding quantification to geometry-aware graph generation. We hope these insights will inform the design of future graph learning methods with stronger geometric consistency.


\begin{appendices}

\section{Proof of Theorem 1}

\begin{proof}
We divide the proof into several steps.

\paragraph{Step 1: Construct disjoint balls around layer-$l$ embeddings.}

For each node $v_i\in\mathcal V_l$, define the open Euclidean ball
\begin{equation}
\small
B_i := B\!\left(f(v_i),\frac{\delta_l}{2}\right)
=
\left\{
h\in\mathbb R^d:
\|h-f(v_i)\|_2 < \frac{\delta_l}{2}
\right\}.
\label{eq:def_B_i}
\end{equation}

We claim that the family $\{B_i\}_{i=1}^{N_l}$ is pairwise disjoint.
Indeed, suppose for contradiction that there exists a point
$h\in B_i\cap B_j$ for some $i\neq j$.
Then, by the triangle inequality,
\begin{equation}
\small
\|f(v_i)-f(v_j)\|_2
\le
\|f(v_i)-h\|_2 + \|h-f(v_j)\|_2 
<
\delta_l.
\label{eq:contradiction_overlap}
\end{equation}
This contradicts the definition of $\delta_l$. 
Hence,
\begin{equation}
B_i \cap B_j = \varnothing
\qquad
\text{for all } i\neq j.
\label{eq:disjointness}
\end{equation}

Therefore, since the balls are pairwise disjoint,
\begin{equation}
\mathrm{Vol}\!\left(\bigcup_{i=1}^{N_l} B_i\right)
=
\sum_{i=1}^{N_l}\mathrm{Vol}(B_i).
\label{eq:volume_sum_disjoint}
\end{equation}

Let
\[
\overline V_d := \mathrm{Vol}(B(0,1))
\]
denote the volume of the unit ball in $\mathbb R^d$.
Using the standard scaling law for Euclidean balls,
\[
\mathrm{Vol}(B(0,R)) = \overline V_d\,R^d,
\]
we have
\begin{equation}
\mathrm{Vol}(B_i)
=
\mathrm{Vol}\!\left(B\!\left(0,\frac{\delta_l}{2}\right)\right)
=
\overline V_d\left(\frac{\delta_l}{2}\right)^d.
\label{eq:vol_Bi}
\end{equation}
Substituting \eqref{eq:vol_Bi} into \eqref{eq:volume_sum_disjoint} yields
\begin{equation}
\mathrm{Vol}\!\left(\bigcup_{i=1}^{N_l} B_i\right)
=
N_l\,\overline V_d \left(\frac{\delta_l}{2}\right)^d.
\label{eq:union_ball_volume}
\end{equation}

\paragraph{Step 2: Bound the radial extent of the $l$-th layer by $A_l$.}

Fix any node $v\in\mathcal V_l$.
Let
\[
v_0=\mathrm{root},\ v_1,\ \dots,\ v_l=v
\]
be a root-to-$v$ path.
By repeated application of the triangle inequality,

\begin{equation}
\small
\begin{aligned}
\|f(v)\|_2
&\le \|f(v_l)-f(v_0)\|_2 + \|f(v_0)\|_2 \\[2pt]
&\le \sum_{i=1}^{l} \|f(v_i)-f(v_{i-1})\|_2 + \|f(v_0)\|_2.
\end{aligned}
\label{eq:path_triangle}
\end{equation}
Since translation does not change the internal structure, we may assume without loss of generality that the embedding of the root is the origin, i.e., $f(\mathrm{v_0}) = \mathbf{0}$. 

Now, for each edge $(v_{i-1},v_i)\in\mathcal E_i$, the definition of  \(
\alpha_i =
\max_{(u,v)\in\mathcal E_i}
\frac{\|f(u)-f(v)\|_2}{d_i},
\)  implies
\begin{equation}
\|f(v_i)-f(v_{i-1})\|_2 \le \alpha_i d_i.
\label{eq:edge_stretch_bound}
\end{equation}
Substituting \eqref{eq:edge_stretch_bound} into \eqref{eq:path_triangle} gives
\begin{equation}
\|f(v)\|_2
\le
\sum_{i=1}^{l} \alpha_i d_i
=
A_l.
\label{eq:radial_bound}
\end{equation}
Since $v\in\mathcal V_l$ was arbitrary, every embedded node in the $l$-th layer lies in the Euclidean ball
\begin{equation}
f(\mathcal V_l)\subset B(0,A_l).
\label{eq:layer_inside_A_l}
\end{equation}

\paragraph{Step 3: Bound the union of small balls by one larger ball.}

Take any $h\in B_i$.
By definition of $B_i$ and \eqref{eq:radial_bound},
\begin{equation}
\|h\|_2
\le
\|h-f(v_i)\|_2 + \|f(v_i)\|_2
<
\frac{\delta_l}{2} + A_l.
\label{eq:h_inside_big_ball}
\end{equation}
Hence,
\begin{equation}
B_i \subset B\!\left(0,A_l+\frac{\delta_l}{2}\right)
\qquad
\text{for all } i.
\label{eq:each_ball_inside_big_ball}
\end{equation}
Therefore,
\begin{equation}
\bigcup_{i=1}^{N_l} B_i
\subset
B\!\left(0,A_l+\frac{\delta_l}{2}\right).
\label{eq:union_inside_big_ball}
\end{equation}
Taking volumes and using monotonicity of Lebesgue measure, we obtain
\begin{equation}
\mathrm{Vol}\!\left(\bigcup_{i=1}^{N_l} B_i\right)
\le
\mathrm{Vol}\!\left(B\!\left(0,A_l+\frac{\delta_l}{2}\right)\right)
=
\overline V_d \left(A_l+\frac{\delta_l}{2}\right)^d.
\label{eq:upper_union_volume}
\end{equation}

\paragraph{Step 4: Compare volumes.}

Combining \eqref{eq:union_ball_volume} and \eqref{eq:upper_union_volume}, we get
\begin{equation}
N_l\,\overline V_d \left(\frac{\delta_l}{2}\right)^d
\le
\overline V_d \left(A_l+\frac{\delta_l}{2}\right)^d.
\label{eq:volume_compare}
\end{equation}
Canceling the positive constant $\overline V_d$ gives
\begin{equation}
N_l \left(\frac{\delta_l}{2}\right)^d
\le
\left(A_l+\frac{\delta_l}{2}\right)^d.
\label{eq:volume_compare_simple}
\end{equation}

Next, we show that \(\delta_l \le 2A_l\).
Indeed, since all points \(f(v_i)\) lie in \(B(0,A_l)\), any two of them satisfy
\begin{equation}
\|f(v_i)-f(v_j)\|_2 \le 2A_l.
\end{equation}
Taking the minimum over \(i\neq j\) gives
\begin{equation}
\delta_l \le 2A_l.
\label{eq:delta_le_2A}
\end{equation}
Substituting \eqref{eq:delta_le_2A} into the right-hand side of \eqref{eq:volume_compare_simple}, we obtain
\[
A_l+\frac{\delta_l}{2}\le 2A_l.
\]
Therefore,
\begin{equation}
N_l \left(\frac{\delta_l}{2}\right)^d
\le
(2A_l)^d.
\label{eq:volume_compare_final}
\end{equation}

Taking the \(d\)-th root of both sides yields
\[
N_l^{1/d}\frac{\delta_l}{2} \le 2A_l,
\]
and hence
\begin{equation}
\delta_l \le 4A_l\,N_l^{-1/d}.
\label{eq:delta_A_N}
\end{equation}

\paragraph{Step 5: Specialize to exponential layer growth.}

For the depth-$l$ layer $\mathcal V_l$, let $k_i$ denote the average branching factor from depth $i-1$ to depth $i$. Then the size of $\mathcal V_l$ is exactly
\begin{equation}
|\mathcal V_l| = N_l = \prod_{i=1}^l k_i .
\end{equation}
Defining
\(k = \left(\prod_{i=1}^l k_i\right)^{1/l},\)
we may equivalently write
\begin{equation}
|\mathcal V_l| = N_l = k^l .
\end{equation}
 Substituting this into \eqref{eq:delta_A_N} immediately gives
\begin{equation}
\delta_l \le CA_l\,k^{-l/d}= C(\sum_{i=1}^l \alpha_i d_i)\,k^{-l/d},
\label{whole-eq}
\end{equation}
which proves \eqref{crowding_theorem}.

Consequently, if the Euclidean embedding preserves the radial tree
distance and the edge weights $d_i$ along the tree remain unchanged,
then
\[
\delta_l \le CA_l\,k^{-l/d}
           = C\!\left(\sum_{i=1}^{l} d_i\right) k^{-l/d}
           = Cd_1(l\,k^{-l/d}).
\]

For fixed and moderate embedding dimensions, the bound implies that
$\delta_l$ decreases with depth roughly at the rate $k^{-l/d}$, up to a polynomial factor in $l$.
Moreover, the exponent $1/d$ suggests that increasing the embedding dimension $d$ can partially
mitigate this crowding effect.


\paragraph{Step 6: Stretch--squeeze trade-off.}

Conversely, suppose one requires that the same-layer separation
remains bounded away from zero, i.e.,
\[
\delta_l \ge m > 0 .
\]
Then \eqref{whole-eq} implies

\[
m
\le
C\,A_l\,k^{-l/d},
\]

which yields

\begin{equation}
A_l \ge \frac{m}{C}\,k^{l/d}.
\end{equation}

This reveals an inherent geometric tension in Euclidean embeddings of
hierarchical graphs: maintaining sufficient separation between nodes
within the same layer requires exponentially increasing radial
distances, whereas constraining radial growth inevitably leads to
crowding within layers. This tension constitutes the stated
\emph{stretch--squeeze trade-off}.
\end{proof}

\section{Proof of Lemma 1}

\begin{proof}

Following standard margin-based learning theory, the functional margin is defined on the real-valued score $s(u,v) \in \mathbb{R}$ rather than its sigmoid-transformed probability. The sigmoid transformation maps real-valued scores into the bounded interval $(0,1)$, which compresses large score values and distorts the relative magnitudes of margins. 

Recall that the functional margin is defined as
\begin{equation}
\hat\gamma 
:= 
\min_{(u,v)\in\mathcal D^+\cup\mathcal D^-} 
y_{uv}\, s(u,v),
\qquad
s(u,v)=\frac{r-d(u,v)}{t},
\end{equation}
where $y_{uv}=+1$ for $(u,v)\in\mathcal D^+$ and $y_{uv}=-1$ for $(u,v)\in\mathcal D^-$.

For any positive pair $(u,v)\in\mathcal D^+$, by definition of $d_+$ we have $d(u,v)\le d_+$, and therefore
\begin{equation}
y_{uv}s(u,v)
=
\frac{r-d(u,v)}{t}
\ge
\frac{r-d^+}{t}.
\end{equation}
Similarly, for any negative pair $(u,v)\in\mathcal D^-$, by definition of $d_-$ we have $d(u,v)\ge d_-$, which implies
\begin{equation}
y_{uv}s(u,v)
=
-\frac{r-d(u,v)}{t}
=
\frac{d(u,v)-r}{t}
\ge
\frac{d^- - r}{t}.
\end{equation}
Combining the two cases yields
\begin{equation}
\hat\gamma(r)=\min\!\left\{
\frac{r-d^+}{t},
\frac{d^- - r}{t}
\right\} .
\end{equation}
The right-hand side is a concave, piecewise-linear function of $r$, whose maximum is attained when the two terms are equal, namely at $r^\star = \frac{d_- + d_+}{2}.$

Substituting $r^\star$ into the bound above gives
\begin{equation}
\sup_{r\in\mathbb R}\hat\gamma
=\hat\gamma(r^\star)=
\frac{d_- - d_+}{2t}
=\frac{\Delta}{2t},
\end{equation}
where $\Delta := d_- - d_+$ denotes the geometric separation between negative and positive pairs.

\end{proof}

\section{Proof of Theorem 2}

\begin{proof}
\label{app:gen_bound}


Following Bartlett and Mendelson's Rademacher-complexity framework\cite{bartlett2002rademacher}, we obtain a margin-based generalization bound by applying a margin surrogate $\phi_\gamma$ (the ramp loss) to the pairwise signed scores $y_{uv}s(u,v)$.
The ramp loss is $1/\gamma$-Lipschitz and dominates the $0$--$1$ loss, i.e., $\phi(z) \ge \mathbf{1}(z \le 0)$, making it a valid surrogate for margin-based generalization analysis.

\begin{figure}[h]
	\centering
	\includegraphics[width=\linewidth]{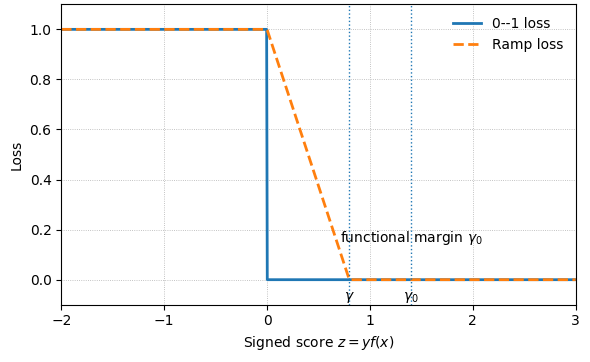}
\caption{
\textbf{Ramp Loss vs. 0--1 Loss.}
The ramp loss serves as a $1/\gamma$-Lipschitz surrogate that upper-bounds the $0$--$1$ loss.
The parameter $\gamma$ controls the slope of the loss, while $\gamma_0$ denotes the functional margin.
}
\end{figure}
We start from Theorem~7 of \cite{bartlett2002rademacher}, which states the following.

Let $\mathcal F$ be a class of real-valued functions defined on data-label pairs,
and let $\phi:\mathbb R\to[0,1]$ satisfy
(i) $\phi(\alpha)\ge \mathbf 1[\alpha\le 0]$ for all $\alpha$, and
(ii) $\phi$ is Lipschitz with constant $L$.
Then, with probability at least $1-\delta$,

\begin{equation}
\label{eq:bm_thm7}
\begin{split}
\mathbb{P}(y_{uv} f(u,v) \le 0)
\le \frac{1}{n} \sum_{i=1}^{n} \phi\big(y_{u_i v_i} f(u_i, v_i)\big) \\
+ 2L \mathfrak{R}_n(\mathcal{F})
+ \sqrt{\frac{\log(2/\delta)}{2n}}.
\end{split}
\end{equation}

Here, $\mathbb P(y_{uv} f(u,v)\le 0)$ denotes the population pairwise misclassification
probability, i.e., the true generalization error.

Fix any $\gamma>0$ and define the ramp  surrogate
\[
\phi_\gamma(\alpha)=
\begin{cases}
1, & \alpha\le 0,\\
1-\alpha/\gamma, & 0<\alpha\le \gamma,\\
0, & \alpha>\gamma.
\end{cases}
\]
By construction, $\phi_\gamma(\alpha)\ge \mathbf 1(\alpha\le 0)$ for all $\alpha$.
Moreover, $\phi_\gamma$ is $1/\gamma$-Lipschitz and satisfies the assumptions of
Theorem~\eqref{eq:bm_thm7} with $L=1/\gamma$.

Let $s(u,v)$ be a pairwise score and define the function class
\[
\mathcal F := \{(u,v)\mapsto y_{uv}\,s(u,v)\}.
\]
Applying Theorem~\eqref{eq:bm_thm7} to $f(u,v)=s(u,v)$ and $\phi=\phi_\gamma$
yields, with probability at least $1-\delta$,
\begin{equation}
\label{eq:surrogate_bound}
\small
\mathcal R_{0\text{-}1}(s)
\;\le\;
\frac1n\sum_{i=1}^n \phi_\gamma\big(y_{u_i v_i} s(u_i,v_i)\big)
\;+\;
\frac{2}{\gamma}\,\mathfrak R_n(\mathcal F)
\;+\;
\sqrt{\frac{\log(2/\delta)}{2n}},
\end{equation}
where $\mathcal R_{0\text{-}1}(s)=\mathbb P\big(y_{uv}s(u,v)\le 0\big)$ denotes the generalization error (population $0$--$1$ risk) of the pairwise classifier.

For each sample $(u_i,v_i)$, we have \(
\phi_\gamma\big(y_{u_i v_i} s(u_i,v_i)\big)
\;\ge\;
\mathbf 1\big(y_{u_i v_i} s(u_i,v_i)\le \gamma\big),
\)
since $\phi_\gamma$ is nonzero only when the functional margin does not exceed $\gamma$.
Averaging over samples yields

\begin{equation}
\small
\frac{1}{n}\sum_{i=1}^n \phi_\gamma\!\big(y_{u_i v_i} s(u_i,v_i)\big)
\le
\frac{1}{n}\sum_{i=1}^n
\mathbf{1}\!\big(y_{u_i v_i} s(u_i,v_i)\le \gamma\big)
= \widehat{\mathcal R}_\gamma(s).
\end{equation}

Substituting this inequality into \eqref{eq:surrogate_bound} yields
\begin{equation}
\label{eq:gen_margin_bound}
\mathcal R_{0\text{-}1}(s)
\;\le\;
\widehat{\mathcal R}_\gamma(s)
\;+\;
\frac{2}{\gamma}\,\mathfrak R_n(\mathcal F)
\;+\;
\sqrt{\frac{\log(2/\delta)}{2n}}.
\end{equation}

By Lemma~\ref{lem:sep_to_margin}, the maximal achievable pairwise functional margin satisfies
\begin{equation}
\label{eq:gamma_hat}
\hat\gamma
\;:=\;
\sup_{r\in\mathbb R}\min_{(u,v)} y_{uv}s(u,v)
\;=\;
\frac{d^- - d^+}{2t}
\;=\;
\frac{\Delta}{2t}.
\end{equation}
The inner minimum gives the tightest margin guarantee for a fixed $r$, while the outer supremum optimizes over $r$ to obtain the largest achievable functional margin.

For near-tree structured data, hyperbolic metrics admit low-distortion  representations of hierarchical distances with near-isometric guarantees\cite{distortion_lilun,low_distortion_lilun}, whereas theorem \textbf{4.1} show that Euclidean representations suffer from exponential crowding with depth, which forces negative pairs to become progressively closer in the ambient space.  Consequently, there exist 
\begin{equation}
\label{eq:delta_compare}
\Delta_{\mathbb R}
\;<\;
\Delta_{\mathbb H}.
\end{equation}
Combining \eqref{eq:gamma_hat} and \eqref{eq:delta_compare} yields a strict margin advantage:
\begin{equation}
\hat\gamma_{\mathbb R}
=
\frac{\Delta_{\mathbb R}}{2t}
\;<\;
\frac{\Delta_{\mathbb H}}{2t}
=
\hat\gamma_{\mathbb H}.
\end{equation}


Since the complexity penalty in \eqref{eq:gen_margin_bound} scales as $O(1/\gamma)$,
the larger margin achieved in hyperbolic space can lead to tighter
margin-based generalization guarantees than those in Euclidean space.

\end{proof}
\end{appendices}

\vfill
\bibliographystyle{IEEEtran} 
\bibliography{myref}         

 
%

\end{document}